%% file: arxiv.tex
\documentclass[12pt]{colt2018}

\usepackage{hyperref}
\usepackage{microtype} 
\usepackage{algorithm}
\usepackage{algpseudocode}

\usepackage{mathtools}
\usepackage{amsmath}
\usepackage{bbm}
\usepackage{amsfonts}
\usepackage{amssymb}
\usepackage{wrapfig}
\usepackage{fullpage}
\usepackage{graphicx}
\usepackage[mathscr]{euscript}

\let\vec\undefined
\usepackage{MnSymbol}
\DeclareMathAlphabet\mathbb{U}{msb}{m}{n}
\usepackage{xpatch}

\DeclareMathOperator{\card}{Card}

\def\Rset{\mathbb{R}}

\DeclareMathOperator*{\E}{\mathbb E}

\DeclareMathOperator{\fat}{fat}
\let\Pr\relax 
\DeclareMathOperator*{\Pr}{\mathbb{P}}

\newcommand{\conf}[1]{}

\newcommand{\Sm}{\mathbb{S}}

\newcommand{\cN}{\mathcal{N}}

\newcommand{\sC}{{\mathscr C}}
\newcommand{\sD}{{\mathscr D}}
\newcommand{\sF}{{\mathscr F}}
\newcommand{\sG}{{\mathscr G}}
\newcommand{\sH}{{\mathscr H}}

\newcommand{\sK}{{\mathscr K}}
\newcommand{\sL}{{\mathscr L}}

\newcommand{\sU}{{\mathscr U}}
\newcommand{\sX}{{\mathscr X}}
\newcommand{\sY}{{\mathscr Y}}
\newcommand{\sZ}{{\mathscr Z}}

\newcommand{\bR}{{\mathbf R}}

\newcommand{\bW}{{\mathbf W}}

\newcommand{\bw}{{\mathbf w}}
\newcommand{\bx}{{\mathbf x}}

\newcommand{\R}{\mathfrak R}

\newcommand{\bsigma}{{\boldsymbol \sigma}}

\newcommand{\A}{\mathfrak r}

\newcommand{\h}{\widehat}

\newcommand{\wt}{\widetilde}
\newcommand{\e}{\epsilon}
\newcommand{\ignore}[1]{}

\makeatletter
\newtheorem*{rep@theorem}{\rep@title}
\newcommand{\newreptheorem}[2]{%
\newenvironment{rep#1}[1]{%
 \def\rep@title{#2 \ref{##1}}%
 \begin{rep@theorem}}%
 {\end{rep@theorem}}}
\makeatother

\hypersetup{
  colorlinks   = true,
  urlcolor     = blue,
  linkcolor    = blue,
  citecolor   = blue
}

\newreptheorem{theorem}{Theorem}
\newreptheorem{lemma}{Lemma}
\newreptheorem{corollary}{Corollary}
\newreptheorem{proposition}{Proposition}

\renewcommand{\L}{\sL}
\title{Relative Deviation Margin Bounds}

\coltauthor{
\AND
\Name{Corinna Cortes}\Email{corinna@google.com}\\
\addr Google Research
\AND
\Name{Mehryar Mohri}\Email{mohri@google.com}\\
\addr Google Research and 
Courant Institute of Mathematical Sciences, New York
\AND
\Name{Ananda Theertha Suresh}\Email{theertha@google.com}\\
\addr Google Research, New York
}

\begin{document}
\maketitle
\input{content.tex}

\end{document}

%% file: content.tex
\begin{abstract}
We present a series of new and more favorable margin-based learning
guarantees that depend on the empirical margin loss of a predictor.
We give two types of learning bounds, both distribution-dependent and
valid for general families, in terms of the Rademacher complexity or
the empirical $\ell_\infty$ covering number of the hypothesis set
used.  Furthermore, using our relative deviation margin bounds, we derive
distribution-dependent generalization bounds for unbounded loss
functions under the assumption of a finite moment.  We also briefly
highlight several applications of these bounds and discuss their
connection with existing results.
\end{abstract}

\section{Introduction}

Margin-based learning bounds provide a fundamental tool for the
analysis of generalization in classification
\citep{Vapnik1998,Vapnik2006,SchapireFreundBartlettLee1997,KoltchinskiiPanchenko2002,TaskarGuestrinKoller2003,BartlettShaweTaylor1998}. These
are guarantees that hold for real-valued functions based on the notion
of confidence margin. Unlike worst-case bounds based on standard
complexity measures such as the VC-dimension, margin bounds provide
optimistic guarantees: a strong guarantee holds for predictors
that achieve a relatively small empirical margin loss,
for a relatively large
value of the confidence margin.
More generally, guarantees similar to margin bounds can be derived
based on notion of a luckiness
\citep{ShaweTaylorBartlettWilliamsonAnthony1998,
  KoltchinskiiPanchenko2002}.

Notably, margin bounds do not have an explicit dependency on the
dimension of the feature space for linear or kernel-based hypotheses.
They provide strong guarantees for
large-margin maximization algorithms such as Support Vector Machines
(SVM) \citep{CortesVapnik1995}, including when used for positive
definite kernels such as Gaussian kernels, for which the dimension of
the feature space is infinite. Similarly, margin-based learning
bounds have helped derive significant guarantees for AdaBoost
\citep{FreundSchapire1997,SchapireFreundBartlettLee1997}.  More
recently, margin-based learning bounds have been derived for neural
networks (NNs) \citep{NeyshaburTomiokaSrebro2015,BartlettFosterTelgarsky2017} and convolutional
neural networks (CNNs) \citep{LongSedghi2020}.

An alternative family of tighter learning guarantees is that of
relative deviation bounds
\citep{Vapnik1998,Vapnik2006,AnthonyShaweTaylor1993,CortesGreenbergMohri2019}. These
are bounds on the difference of the generalization and empirical error
scaled by the square-root of the generalization error or empirical
error, or some other power of the error. The scaling is similar to
dividing by the standard deviation since, for smaller values of the
error, the variance of the error of a predictor roughly coincides with
its error. These guarantees translate into very useful bounds on the
difference of the generalization error and empirical error whose
complexity terms admit the empirical error as a factor.

This paper presents general relative deviation {\em margin} bounds.  These
bounds combine the benefit of standard margin bounds and that of
standard relative deviation bounds, thereby resulting in tighter
margin bounds (Section~\ref{sec:H-applications}).  
As an example, our learning bounds provide 
tighter guarantees for margin-based algorithms such as SVM and
boosting than existing ones.
We give two families of relative deviation 
bounds, both distribution-dependent and valid
for general hypothesis sets. Additionally, both families
of guarantees hold for an arbitrary $\alpha$-moment,
with $\alpha \in (1, 2]$.
In Section~\ref{sec:applications}, we
also briefly highlight several applications of our bounds and
discuss their connection with existing results.

Our first family of margin bounds are expressed in terms
of the empirical $\ell_\infty$-covering
number of the hypothesis set (Section~\ref{sec:covering}).
We show how these empirical covering numbers can be
upper bounded to derive empirical fat-shattering guarantees.
One benefit of these resulting guarantees is that there are
known upper bounds for various standard hypothesis sets,
which can be leveraged to derive explicit bounds (see Section~\ref{sec:H-applications}).

Our second family of margin bounds are expressed in terms 
of the Rademacher complexity of the hypothesis set used (Section~\ref{sec:rademacher}). Here, our learning bounds are first expressed in terms of a peeling-based Rademacher complexity term
we introduce. Next, we give a series of upper bounds on
this complexity measure, first simpler ones in terms of Rademacher
complexity, next in terms of empirical $\ell_2$ covering numbers,
and finally in terms of the so-called \emph{maximum Rademacher complexity}.
In particular, we show that a simplified version of our bounds yields a 
guarantee similar to the maximum Rademacher margin bound of \cite{SrebroSridharanTewari2010}, but with
more favorable constants and for a general $\alpha$-moment. 

We then use these family of margin bounds for $\alpha$-moments to provide generalization guarantees for unbounded loss functions (Section~\ref{sec:unbound}). We also illustrate these results by deriving explicit bounds for various standard hypothesis sets in Section~\ref{sec:applications}. In the next sub-section, we further highlight 
our contributions and compare them to the previous work.

\subsection{Previous work and our contributions}

\ignore{
We have three contributions in our paper: $\ell_\infty$ based bounds, Rademacher complexity bounds, and implications for unbounded losses. Below, we further
highlight our contributions and compare them to the previous work 
for each of the above results.

Relative deviation margin bounds have been studied by several authors in the literature. 
}

\textbf{$\ell_\infty$-covering based bounds:} A version of our main
result for empirical $\ell_\infty$-covering number bounds in the
special case $\alpha \!=\!2$ was postulated by \cite{Bartlett1998}
without a proof. The author suggested that the proof could be given by
combining various techniques with the results of
\cite{AnthonyShaweTaylor1993} and \cite{Vapnik1998,
Vapnik2006}. However, as pointed out by
\cite{CortesGreenbergMohri2019}, the proofs given by
\cite{AnthonyShaweTaylor1993} and \cite{Vapnik1998, Vapnik2006} are
incomplete and rely on a key lemma that is not
proven. 
\ignore{
Nevertheless, our proof and presentation 
partly benefit from the analysis of \cite{Bartlett1998}, in particular
the bound on the covering number (Corollary~\ref{cor:3}).
}
\cite{zhang2002covering} studied covering number-based
non-relative bounds for linear classifiers but postulated that his
techniques could be modified, using Bernstein-type concentration bounds,
to obtain relative deviation $\ell_\infty$-covering number bounds for
linear classifiers. However, a careful inspection suggests that this
is not a straightforward exercise and obtaining such bounds in fact
requires techniques such as those we use in this paper, or, perhaps,
somewhat similar ones.
\textit{Our contribution:} We provide a self-contained 
proof based on a margin-based symmetrization argument. 
The proof technique uses 
a new symmetrization argument that is different from those of 
\cite{Bartlett1998} and \cite{zhang2002covering}.

\textbf{Rademacher complexity bounds:} Using ideas from local Rademacher complexity~\citep{bartlett2005local}, Rademacher complexity bounds were given in \cite{SrebroSridharanTewari2010}, however their bounds are based on the so-called 
\emph{maximum Rademacher complexity}, 
which depends on the worst possible sample and is therefore independent 
of the underlying distribution.
\textit{Our contribution:} We provide the first 
distribution-dependent relative deviation margin bounds, 
in terms of a peeling-based Rademacher complexity. 
The proof is based on the a new peeling-based arguments, 
which were not known before. 
Finally, we show that we
can recover the bounds of \cite{SrebroSridharanTewari2010} with 
more favorable constants.

\textbf{Generalization bounds for unbounded loss functions}: Commonly
used loss functions such as cross-entropy are unbounded and thus standard
relative deviation bounds do not hold for them. \cite{CortesGreenbergMohri2019}
provided zero-one relative deviation bounds which they used to derive
bounds for unbounded losses, in terms of the discrete dichotomies generated
by the hypothesis class, under the assumption of a finite moment of
the loss.
\textit{Our contribution}: We present the first generalization bounds
for unbounded loss functions in terms of covering numbers and
Rademacher complexity, which are \emph{optimistic bounds} that, in
general, are more favorable than the previous known bounds of
\cite{CortesGreenbergMohri2019}, under the same finite moment
assumption. Doing so required us to derive relative deviation margin
bounds for general $\alpha$-moment ($\alpha \in (1, 2]$), in contrast
with previous work, which only focused on the special case $\alpha =
2$.

Recently, relative deviation margin bounds for the special case of
linear classifiers were studied by \cite{gronlund2020near}. Both their
results and the proof techniques are specific to linear classifiers. In
contrast, our bounds hold for any general hypothesis set and recovers
the bounds of \cite{gronlund2020near} for the special case of linear
classifiers up to logarithmic factors. Relative deviation PAC-Bayesian bounds
were also derived by \cite{mcallester2003simplified} for linear hypothesis
sets.
It is known, however, that Rademacher complexity learning bounds
are finer guarantees since, as shown recently
by \cite{KakadeSridharantTewari2008} and 
\cite{foster2019hypothesis}, they can be used to 
derive finer PAC-Bayesian guarantees than previously
known ones.
\ignore{
In view of that, we suspect that our Rademacher complexity 
bounds are more favorable, at least in the linear case. 
}

\ignore{
While this is in a
different learning scenario, based on recent results on relationship
between Rademacher bounds and PAC Bayesian
bounds~\citep{foster2019hypothesis}, we suspect our bounds are better,
at least in the linear case. 
}

\ignore{
\textbf{Novelty and proof techniques.}
A version of our main result for empirical $\ell_\infty$-covering number
bounds for the special case $\alpha \!=\!2$ was postulated by
\cite{Bartlett1998} without a proof. The author suggested that the
proof could be given by combining various techniques with the results of
\cite{AnthonyShaweTaylor1993} and \cite{Vapnik1998,
  Vapnik2006}. However, as pointed out by
\cite{CortesGreenbergMohri2019}, the proofs given by
\cite{AnthonyShaweTaylor1993} and \cite{Vapnik1998, Vapnik2006} are
incomplete and rely on a key lemma that is not proven.
Our proof and presentation follow \citep{CortesGreenbergMohri2019} but
also partly benefit from the analysis of \cite{Bartlett1998}, in particular
the bound on the covering number (Corollary~\ref{cor:3}).
To the best of our knowledge, our Rademacher
complexity learning bounds of Section~\ref{sec:rademacher} are new. The proof consists
of using a peeling technique combined with an application
of a bounded difference inequality finer than McDiarmid's 
inequality. For both families of bounds, the proof relies on
a margin-based symmetrization result (Lemma~\ref{lem:cover}) proven in the next section.
}

\section{Symmetrization}
\label{sec:symmetrization}

In this section, we prove two key symmetrization-type lemmas
for a relative deviation between the expected binary
loss and empirical margin loss.

We consider an input space $\sX$ and a binary output space
$\sY = \set{-1, +1}$ and a hypothesis set $\sH$
of functions mapping from $\sX$ to $\Rset$.
We denote by $\sD$ a distribution over
$\sZ = \sX \times \sY$ and denote by $R(h)$ the generalization
error and by $\h R_S(h)$ the empirical error of a hypothesis
$h \in \sH$:
\begin{equation}
R(h) = \E_{z = (x, y) \sim \sD}[1_{yh(x) \leq 0}],
\qquad \h R_{S}(h) = \E_{z = (x, y) \sim S}[1_{y h(x) \leq 0}],
\end{equation}
where we write $z \sim S$ to indicate that $z$ is randomly drawn 
from the empirical distribution defined by $S$. Given $\rho \geq 0$,
we similarly defined the $\rho$-margin loss and empirical 
$\rho$-margin loss of $h \in \sH$:
\begin{equation}
R^\rho(h) = \E_{z = (x, y) \sim \sD}[1_{yh(x) < \rho}],
\qquad \h R^\rho_{S}(h) = \E_{z = (x, y) \sim S}[1_{yh(x) < \rho}] .
\end{equation}
We will sometimes use the shorthand ${x_{1}^{m}}$ to denote a sample
of $m$ points $(x_{1}, \ldots, x_{m}) \in \sX^{m}$.

The following is our first symmetrization lemma in terms of empirical margin loss.
The parameter $\tau > 0$ is used to ensure a positive denominator
so that the relative deviations are mathematically well defined.
\begin{lemma}
\label{lem:four}
Fix $\rho \geq 0$ and $1 < \alpha \leq 2$ and assume that
$m \e^{\frac{\alpha}{\alpha - 1}} > 1$. Then, for any any
$\e, \tau > 0$, the following inequality holds:
\begin{equation*}
  \Pr_{S \sim \sD^{m}} \left[\sup_{h \in \sH} \frac{R(h) - \h
      R^\rho_{S}(h)}{\sqrt[\alpha]{R(h) + \tau}} > \e \right] 
  \leq 4 \Pr_{S, S^{\prime} \sim \sD^{m}} \left[\sup_{h \in \sH} \frac{ \h
      R_{S^{\prime}}(h) - \h R^\rho_{S}(h) }{\sqrt[\alpha]{\frac{1}{2} [ \h R_{S^{\prime}}(h) + \h R^\rho_{S}(h) +\frac{1}{m}] } } > \e \right].
\end{equation*}
\end{lemma}
The proof is presented in Appendix~\ref{app:symmetrization}. It consists of extending the proof technique of \cite{CortesGreenbergMohri2019} for standard
empirical error to the empirical margin case
and of using the binomial inequality \citep[Lemma~\ref{lem:binomial}]{GreenbergMohri2013}.
The lemma helps us bound the relative deviation in terms of the empirical margin loss on a sample $S$ and the empirical error on an independent sample $S'$, both of size $m$.

\begin{figure}[t]
\centering
\includegraphics[scale=0.4]{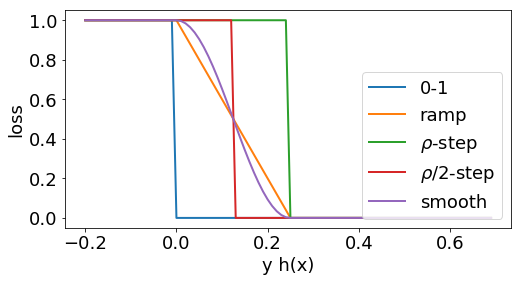}
\caption{Illustration of different choices of function $\phi$ for $\rho = 0.25$.}
\label{fig:margin}
\end{figure}

We now introduce some notation needed for the presentation and
discussion of our relative deviation margin bound. Let $\phi\colon \Rset \to \Rset_+$ be a function such that the following inequality holds for all $x \in \Rset$:
\[
1_{x <0} \leq \phi(x) \leq 1_{x < \rho}.
\]
As an example, we can choose $\phi(x) = 1_{x < \rho/2}$ as in the previous sections. For a sample $z = (x,y)$, let $g(z) = \phi(yh(x))$. Then,
\begin{equation}
\label{eq:margin-loss}
1_{yh(x) <0} \leq g(z) \leq 1_{yh(x) < \rho}.
\end{equation}
Let the family $\sG$ be defined as follows: $\sG = \{z = (x, y) \mapsto \phi(yh(x)) \colon h \in \sH \}$ and let $R(g) = \E_{z \sim \sD}[g(z)]$ denote the expectation of $g$ and $\h R_{S}(g) = \E_{z \sim S}[g(z)]$
its empirical expectation for a sample $S$. There are several choices for function $\phi$, as illustrated by Figure~\ref{fig:margin}. For example, $\phi(x)$ can be chosen to be $1_{x < \rho}$ or $1_{x < \rho/2}$ \citep{Bartlett1998}. $\phi$ can also be chosen to be the so-called \emph{ramp loss}:
\begin{equation*}
  \phi(x) =
    \begin{cases}
      1 & \text{if } x < 0 \\
      1 - \frac{x}{\rho} & \text{if }x \in [0, \rho] \\
      0 & \text{if } x > \rho,
    \end{cases}       
\end{equation*}
or the smoothed margin loss chosen by \citep{SrebroSridharanTewari2010}:
\begin{equation*}
  \phi(x) =
    \begin{cases}
      1 & \text{if } x< 0 \\
      \frac{1 + \cos(\pi x/ \rho)}{2} & \text{if } x \in [0, \rho] \\
      0 & \text{if } x > \rho.
    \end{cases}       
\end{equation*}
Fix $\rho > 0$. Define the $\rho$-truncation function $\beta_\rho
\colon \Rset \to [-\rho, +\rho]$ by
$\beta_\rho(u) = \max\set{u, -\rho} 1_{u \leq 0} + \min\set{u, +\rho}
1_{u \geq 0}$, for all $u \in \Rset$. For any $h \in \sH$, we denote
by $h_\rho$ the $\rho$-truncation of $h$, $h_\rho = \beta_\rho(h)$,
and define $\sH_\rho = \set{h_\rho \colon h \in \sH}$.

For any family of functions $\sF$, we also denote by
$\cN_\infty(\sF, \e, x_1^m)$ the empirical covering number of $\sF$
over the sample $(x_1, \ldots, x_m)$ and by $\sC(\sF, \e, x_1^m)$
a minimum empirical cover. Then, the following symmetrization lemma holds.
\begin{lemma}
\label{lem:cover}
Fix $\rho \geq 0$ and $1 < \alpha \leq 2$. Then, the
following inequality holds:
\begin{equation*}
\Pr_{S, S^{\prime} \sim \sD^{m}} \left[\sup_{h \in \sH} \frac{ \h
R_{S^{\prime}}(h) - \h R^\rho_{S}(h) }{ \sqrt[\alpha]{\frac{1}{2} [ \h
  R_{S^{\prime}}(h) + \h R^\rho_{S}(h) +\frac{1}{m}] } } > \e
\right] \leq
\Pr_{S, S^{\prime} \sim \sD^{m}} \left[\sup_{g \in \sG} \frac{ \h
R_{S^{\prime}}(g) - \h R_{S}(g) }{ \sqrt[\alpha]{\frac{1}{2} [ \h R_{S^{\prime}}(g) + \h R_{S}(g) +\frac{1}{m}] } } > \e \right].
\end{equation*}
Further for $g(z)= 1_{yh(x) < \rho/2}$, using the shorthand $\sK = \sC(\sH_\rho,
    \frac{\rho}{2}, S \cup S')$, the following holds:
\begin{equation*}
\Pr_{S, S^{\prime} \sim \sD^{m}} \left[\sup_{h \in \sH} \frac{ \h
R_{S^{\prime}}(h) - \h R^\rho_{S}(h) }{ \sqrt[\alpha]{\frac{1}{2} [ \h
  R_{S^{\prime}}(h) + \h R^\rho_{S}(h) +\frac{1}{m}] } } > \e
\right] \leq
\Pr_{S, S^{\prime} \sim \sD^{m}} \left[\sup_{h \in \sK} \frac{ \h
R^{\frac{\rho}{2}}_{S^{\prime}}(h) - \h R^{\frac{\rho}{2}}_{S}(h) }{ \sqrt[\alpha]{\frac{1}{2} [ \h R^{\frac{\rho}{2}}_{S^{\prime}}(h) + \h R^{\frac{\rho}{2}}_{S}(h) +\frac{1}{m}] } } > \e \right].
\end{equation*}
\end{lemma}
The proof consists
of using inequality~\ref{eq:margin-loss}, it is given in Appendix~\ref{app:symmetrization}. The first result of the lemma gives an upper bound for a general choice of functions $g$, that is for an arbitrary choices of the $\Phi$ loss function. This inequality will be used in Section~\ref{sec:rademacher} to derive our Rademacher complexity bounds. The second inequality is for the specific choice of $\Phi$ that corresponds to $\rho/2$-step function. We will use this inequality in the next section to derive $\ell_\infty$ covering number bounds.

\section{Relative deviation margin bounds -- Covering numbers}
\label{sec:covering}

In this section, we present a general relative deviation margin-based learning bound, 
expressed in terms of the expected empirical covering number of
$\sH_\rho$. The learning guarantee is thus distribution-dependent. It is also 
very general since it is given for any
$1\!<\! \alpha\! \leq\! 2$ and an arbitrary hypothesis set.

\begin{theorem}[General relative deviation margin bound]
  \label{th:relative} Fix $\rho \geq 0$ and $1 < \alpha \leq 2$. Then,
  for any hypothesis set $\sH$ of functions mapping from $\sX$ to
  $\Rset$ and any $\tau > 0$, the following inequality holds:
\begin{align*}
& \Pr_{S \sim \sD^{m}} \left[\sup_{h \in \sH} \frac{R(h) - \h
  R^\rho_{S}(h)}{\sqrt[\alpha]{R(h) + \tau}} > \e \right] 
\leq 4 \, \E_{x_1^{2m} \sim \sD^{2m}}[\cN_\infty(\sH_\rho, \tfrac{\rho}{2}, x_1^{2m})] \ 
\exp \left[ \frac{-m^{\frac{2 (\alpha - 1)}{\alpha}} \e^2}{2^{\frac{\alpha + 2}{\alpha}}} \right].
\end{align*}
\end{theorem}

The proof is given in Appendix~\ref{app:covering}.
As mentioned earlier, a version of this result for $\alpha = 2$ was postulated by
\cite{Bartlett1998}. The result can be alternatively expressed as follows, taking the limit $\tau \to 0$.

\begin{corollary}
\label{cor:1} 
Fix $\rho \geq 0$ and $1 < \alpha \leq 2$. Then, for any hypothesis
set $\sH$ of functions mapping from $\sX$ to $\Rset$, with
probability at least $1 - \delta$, the following
inequality holds for all $h \in \sH$:
\begin{equation*}
R(h) 
\leq  \h R^\rho_{S}(h) + 2^{\frac{\alpha + 2}{2 \alpha}} \sqrt[\alpha]{R(h)}
  \sqrt{\frac{\log \E[\cN_\infty(\sH_\rho, \tfrac{\rho}{2}, x_1^{2m})] + \log \frac{1}{\delta}}{m^{\frac{2 (\alpha - 1)}{\alpha}}}}.
\end{equation*}
\end{corollary}
Note that a smaller value of $\alpha$ ($\alpha$ closer to $1$) might be advantageous
for some values of $R(h)$, at the price of a worse complexity in terms of the sample size.
For $\alpha = 2$, the result can be rewritten as follows.
\begin{corollary}
\label{cor:2} 
Fix $\rho \geq 0$. Then, for any hypothesis
set $\sH$ of functions mapping from $\sX$ to $\Rset$, with
probability at least $1 - \delta$, the following
inequality holds for all $h \in \sH$:
\begin{align*}
R(h) 
\leq  \h R^\rho_{S}(h) + 2\sqrt{\h R^\rho_{S}(h) \frac{\log
  \E[\cN_\infty(\sH_\rho, \tfrac{\rho}{2}, x_1^{2m})] + \log
  \frac{1}{\delta}}{m}}
+ 4 \frac{\log \E[\cN_\infty(\sH_\rho, \tfrac{\rho}{2}, x_1^{2m})] + \log \frac{1}{\delta}}{m}.
\end{align*}
\end{corollary}
\begin{proof}
  Let $a = R(h)$, $b = \h R^\rho_{S}(h)$, and
  $c = \frac{\log \E[\cN_\infty(\sH_\rho, \tfrac{\rho}{2}, x_1^{2m})),
    \frac{\rho}{2})] + \log \frac{1}{\delta}}{m}$. Then, for
  $\alpha = 2$, the inequality of Corollary~\ref{cor:1} can be
  rewritten as
\[
a \leq b + 2 \sqrt{ca}.
\]
This implies that $(\sqrt{a} - \sqrt{c})^2 \leq b + c$ and hence
$\sqrt{a} \leq \sqrt{b + c} + \sqrt{c}$. Therefore,
$a \leq b + 2c + 2\sqrt{(b+c)c} \leq b + 4c +
2\sqrt{cb}$. Substituting the values of $a, b,$ and $c$ yields the
bound.
\end{proof}
The guarantee just presented provides a tighter margin-based learning bound than standard margin bounds since the dominating term admits the empirical margin loss as a factor. Standard margin bounds are subject to a trade-off: a large value of $\rho$ reduces the complexity term while leading to a larger empirical margin loss term. Here, the presence of the empirical loss factor favors this
trade-off by allowing a smaller choice of $\rho$.
The bound is distribution-dependent 
since it is expressed in terms of the expected covering number 
and it holds for an arbitrary hypothesis set $\sH$.

The learning bounds just presented hold for a fixed value of $\rho$. They can be extended to hold uniformly for all values of $\rho \in [0,1]$, at the price of an additional $\log \log$-term. We illustrate that extension for Corollary~\ref{cor:1}.
\begin{corollary}
\label{cor:4} 
Fix  $1 < \alpha \leq 2$. Then, for any hypothesis
set $\sH$ of functions mapping from $\sX$ to $\Rset$ and any $\rho \in (0,r]$, with
probability $\geq 1 - \delta$, the following
inequality holds for all $h \in \sH$:
\begin{equation*}
R(h) 
\leq  \h R^\rho_{S}(h) + 2^{\frac{\alpha + 2}{2 \alpha}} \sqrt[\alpha]{R(h)}
  \sqrt{\frac{\log \E[\cN_\infty(\sH_\rho, \tfrac{\rho}{4}, x_1^{2m})] + \log \frac{1}{\delta} + \log \log_2 \frac{2r}{\rho}}{m^{\frac{2 (\alpha - 1)}{\alpha}}}}.
\end{equation*}
\end{corollary}
\begin{proof}
For $k \geq 1$, let $\rho_k = r / 2^{k}$ and $\delta_k = \delta / k^2$. For all such $\rho_k$, by Corollary~\ref{cor:1} and the union bound,
\begin{equation*}
R(h) 
\leq  \h R^{\rho_k}_{S}(h) + 2^{\frac{\alpha + 2}{2 \alpha}} \sqrt[\alpha]{R(h)}
  \sqrt{\frac{\log \E[\cN_\infty(\sH_\rho, \tfrac{\rho_k}{2}, x_1^{2m})] + \log \frac{1}{\delta} + 2 \log k}{m^{\frac{2 (\alpha - 1)}{\alpha}}}}.
\end{equation*}
By the union bound, the error probability is most $\sum_k \delta_k = \delta \sum_k (1/k^2) \leq \delta$. For any $\rho \in (0,r]$, there exists a $k$ such that $\rho \in (\rho_k, \rho_{k-1}]$. For this $k$, $\rho \leq \rho_{k-1}  = r/2^{k-1}$. Hence, $k \leq \log_2 (2r/\rho)$. By the definition of margin, for all $h \in \sH$, $\h R^{\rho_k}_{S}(h) \leq \h R^{\rho}_{S}(h)$. Furthermore, as $\rho_k = \rho_{k-1}/2 \geq \rho / 2$, 
$\cN_\infty(\sH_\rho, \frac{\rho_k}{2}, x_1^{2m}) \leq \cN_\infty(\sH_\rho, \frac{\rho}{4}, x_1^{2m})$. Hence, 
for all $\rho \in (0,r]$, 
\begin{equation*}
R(h) 
\leq  \h R^\rho_{S}(h) + 2^{\frac{\alpha + 2}{2 \alpha}} \sqrt[\alpha]{R(h)}
  \sqrt{\frac{\log \E[\cN_\infty(\sH_\rho, \tfrac{\rho}{4}, x_1^{2m})] + \log \frac{1}{\delta} + \log \log_2 \frac{2r}{\rho}}{m^{\frac{2 (\alpha - 1)}{\alpha}}}}.
\end{equation*}
\end{proof}

Our previous bounds can be expressed in terms of the fat-shattering dimension, as illustrated below. Recall that, given $\gamma > 0$, a set of 
points $\sU = \set{u_1, \ldots, u_m}$ is said to be \emph{$\gamma$-shattered} by a family of real-valued functions 
$\sH$ if there exist real numbers $(r_1, \ldots, r_m)$ (witnesses) such that for all binary vectors
$(b_1, \ldots, b_m) \in \set{0, 1}^m$, 
there exists $h \in \sH$ such that:
\[
h(x)  
\begin{cases}
\geq r_i + \gamma & \text{if } b_i = 1;\\
\leq r_i - \gamma & \text{otherwise}.
\end{cases}
\]
The \emph{fat-shattering dimension $\fat_\gamma(\sH)$} of the family
$\sH$ is the cardinality of the largest set $\gamma$-shattered set 
by $\sH$ \citep{AnthonyBartlett99}.
\begin{corollary}
\label{cor:3} 
Fix $\rho \geq 0$. Then, for any hypothesis set $\sH$ of functions
mapping from $\sX$ to $\Rset$ with $d = \fat_{\frac{\rho}{16}}(\sH)$,
with probability at least $1 - \delta$, the following holds for all
$h \in \sH$:
\begin{align*}
R(h) \leq \h R^\rho_{S}(h) + 2\sqrt{\h R^\rho_{S}(h) \frac{1 +  d \log_2(2c^2m)
                                \log_2\frac{2cem}{d} + \log
                                \frac{1}{\delta}}{m}}  
+ \frac{ 1 + d \log_2(2c^2m) \log_2\frac{2cem}{d} + \log \frac{1}{\delta}}{m},
\end{align*}
where $c = 17$.
\end{corollary}
\begin{proof}
By \cite[Proof of theorem 2]{Bartlett1998}, we have
\[
\log \max_{x_1^{2m}} [\cN_\infty(\sH_\rho, \tfrac{\rho}{2}, x_1^{2m})
\leq  1 + d' \log_2(2c^2m) \log_2\frac{2cem}{d'},
\]
where $d' =  \fat_{\frac{\rho}{16}}(\sH_\rho) \leq
\fat_{\frac{\rho}{16}}(\sH) = d$. Upper bounding the expectation by
the maximum completes the proof.
\end{proof}
We will use this bound in Section~\ref{sec:H-applications} to derive explicit guarantees for several standard hypothesis sets.
\section{Relative deviation margin bounds -- Rademacher complexity}
\label{sec:rademacher}

In this section, we present relative deviation margin bounds expressed in terms of the Rademacher complexity of the hypothesis sets. As with the previous section, these bounds are general: they hold for any $1 < \alpha \leq 2$ and arbitrary hypothesis sets.

As in the previous section, we will define the family $\sG$ by
$\sG = \set{\phi(yh(x)) \colon h \in \sH }$, where
$\phi$ is a function such that 
\[
1_{x <0} \leq \phi(x) \leq 1_{x < \rho}.
\]

\subsection{Rademacher complexity-based margin bounds}

We first relate the symmetric relative deviation bound to a quantity similar to the Rademacher average, modulo a rescaling.

\begin{lemma}
\label{lem:to_rad}
Fix $1 < \alpha \leq 2$. Then, the
following inequality holds:
\begin{equation*}
\Pr_{S, S^{\prime} \sim \sD^{m}} \left[\sup_{g \in \sG} \frac{ \h
R_{S^{\prime}}(g) - \h R_{S}(g) }{ \sqrt[\alpha]{\frac{1}{2} [ \h R_{S^{\prime}}(g) + \h R_{S}(g) +\frac{1}{m}] } } > \e \right]
\leq 
2 \Pr_{z_{1}^{m} \sim \sD^{m}, \bsigma} \left[\sup_{g \in \sG} \frac{ \frac{1}{m} \sum_{i = 1}^{m} \sigma_{i} g(z_{i}) }{ \sqrt[\alpha]{\frac{1}{m} [\sum_{i = 1}^{m}
(g(z_{i})) + 1]} } > \frac{\e}{2\sqrt{2}} \right].
\end{equation*}
\end{lemma}
The proof is given in Appendix~\ref{app:rademacher}. 
It consists of introducing Rademacher variables
and deriving an upper bound in terms of the first $m$
points only.

Now, to bound the right-hand side of the Lemma~\ref{lem:to_rad}, we use a
peeling argument, that is we partition $\sG$ into subsets $\sG_k$, give a learning bound for each $\sG_k$, and then take a weighted union bound.  For any non-negative integer $k$ with
$0 \leq k \leq \log_2 m$, let $\sG_k(z_1^m)$ denote the family of
hypotheses defined by
\[
\sG_k(z_1^m) = \set[\bigg]{g \in \sG\colon 2^{k} \leq \Big( \sum_{i =
      1}^{m} g(z_{i}) \Big) + 1 < 2^{k+1}}.
\]
Using the above inequality and a peeling argument, we show the following upper bound expressed in terms of Rademacher complexities.
\begin{lemma}
\label{lem:indic}
Fix $1 < \alpha \leq 2 $ and $z_1^m \in \sZ^m$. Then, the following
inequality holds:
\[
\Pr_{\bsigma} \left[\sup_{g \in \sG} \frac{ \frac{1}{m} \sum_{i = 1}^{m} \sigma_{i} g(z_{i}) }{ \sqrt[\alpha]{\frac{1}{m} [\sum_{i = 1}^{m}
(g(z_{i})) + 1]} } > {\e} \, \Bigg\mid \, z^{m}   \right] \leq 2 \sum_{k =
0}^{\lfloor \log_2 m \rfloor}
\exp \left( \frac{m^2\h \R_m^2(\sG_k(z_1^m))}{2^{k+5}}  - \frac{\e^2}{64\frac{2^{k(1-2/\alpha)}}{m^{2-2/\alpha}}} \right)  1_{\e \leq 2
\left(\frac{2^{k}}{m} \right)^{1-1/\alpha}}.
\]
\end{lemma}
The proof is given in Appendix~\ref{app:rademacher}. Instead of 
applying Hoeffding's bound to each term of the left-hand side
for a fixed $g$ and then
using covering and the union bound to bound the supremum, 
here, we seek to bound the supremum over $\sG$ directly. 
To do so, we use
a bounded difference inequality that leads to a finer result
than McDiarmid's inequality.

Let $\A_m(\sG)$ be defined as the following \emph{peeling-based Rademacher} complexity of $\sG$:
\[
\A_m(\sG) = \sup_{0 \leq k \leq \log_2 (m)} \log \left[ \E_{z_1^m \sim
    \sD^m} \left[\exp \left( \frac{m^2\h
        \R_m^2(\sG_k(z_1^m))}{2^{k+5}} \right) \right] \, \right].
\]
Then, the following is a margin-based relative deviation bound expressed in terms of $\A_m(\sG)$, that is in terms of Rademacher complexities.

\begin{theorem}
\label{thm:rad}
Fix $1 < \alpha \leq 2$. 
Then, with probability at least $1 - \delta$, 
for all hypothesis $h \in \sH$, the following inequality holds:
\begin{align*}
R(h) - \h R^\rho_{S}(h) 
& \leq  16\sqrt{2}  \sqrt[\alpha]{R(h)} \left( \frac{\A_m(\sG) + \log \log m + \log \frac{16}{\delta}}{m}\right)^{1-1/\alpha} \\
& = 16\sqrt{2} \left( \frac{\A_m(\sG) + \log \log m + \log \frac{16}{\delta}}{m}\right)
\left( \frac{mR(h)}{\A_m(\sG) + \log \log m + \log \frac{16}{\delta}}\right)^{1/\alpha} .
\end{align*}
\end{theorem}
Combining the above lemma with Theorem~\ref{thm:rad} yields the following.

\begin{corollary}
\label{cor:rad}
Fix $1 < \alpha \leq 2$ and let $\sG$ be defined as above.  Then, with
probability at least $1 - \delta$, for all hypothesis $h \in \sH$,
\[
R(h) - \h
  R^\rho_{S}(h) \leq  32  \sqrt[\alpha]{  \h R^\rho_{S}(h)}  \left(
    \frac{\A_m(\sG) + \log \log m + \log
      \frac{16}{\delta}}{m}\right)^{1-\frac{1}{\alpha}} \mspace{-10mu}
+ \ 2 (32)^{\frac{\alpha}{\alpha-1}} 
  \left(  \frac{\A_m(\sG) + \log \log m + \log \frac{16}{\delta}}{m}\right).
\]
\end{corollary}
The above result can be extended to hold for all $\alpha$ simultaneously. 

\begin{corollary}
\label{cor:all_alpha}
Let $\sG$ be defined as above. Then, with probability at least $1 - \delta$, for all
hypothesis $h \in \sH$ and $\alpha \in (1, 2]$,
\[
R(h) - \h
  R^\rho_{S}(h) \leq  32\sqrt{2}  \sqrt[\alpha]{R(h) }  \left(  \frac{\A_m(\sG) + \log \log m + \log \frac{16}{\delta}}{m}\right)^{1-1/\alpha}.
\]
\end{corollary}

\subsection{Upper bounds on peeling-based Rademacher complexity}

We now present several upper bounds on $\A_m(\sG)$. We provide proofs for all the results in Appendix~\ref{app:peel}. For any hypothesis
set $\sG$, we denote by $\Sm_{\sG}({x_{1}^{m}})$ the number of
distinct dichotomies generated by $\sG$ over that sample:
\begin{align*}
\Sm_{\sG}({z_{1}^{m}}) 
= \card \left(\set[\Big]{\big(g(z_{1}), \ldots, g(z_{m})\big) 
\colon g \in \sG } \right).
\end{align*}
We note that we do not make any assumptions over range of $\sG$.
\begin{lemma}
\label{lem:card}
If the range of $g$ is in $\{0,1\}$, then the following upper bounds hold on the peeling-based Rademacher complexity of $\sG$:
\[
\A_m(\sG) \leq \frac{1}{8} \log \E_{z_1^m} [\Sm_{\sG}({z_{1}^{m}})].
\]
\end{lemma}
Combining the above result with Corollary~\ref{cor:rad}, improves the
relative deviation bounds of \cite[Corollary
2]{CortesGreenbergMohri2019} for  $\alpha < 2$. In particular, we improve the $\sqrt{\E_{z_1^m} [\Sm_{\sG}({z_{1}^{m}})]}$ term in their bounds to $\left(\E_{z_1^m} [\Sm_{\sG}({z_{1}^{m}})] \right)^{1-1/\alpha}$, which is an improvement for $\alpha < 2$. 

We next upper bound the peeling based Rademacher complexity in terms of the covering number. 
\begin{lemma}
\label{lem:covering}
For a set of hypotheses $\sG$,
\[
\A_m(\sG) \leq \sup_{0 \leq k \leq \log_2 (m)} \log \left[ \E_{z_1^m \sim
    \sD^m} \left[\exp \left\{ \frac{1}{16} \left(1 + \int^{1}_{\frac{1}{\sqrt{m}}}
\log N_2 \left(\sG_k(z_1^m),  \sqrt{\tfrac{2^k}{m}} \, \e, z_1^m \right) \, d \epsilon  \right)\right\} \right] \, \right].
\]
\end{lemma}
One can further simplify the above bound using the smoothed margin loss from~\cite{SrebroSridharanTewari2010}. 
Let the worst case Rademacher complexity be defined as follows.
\[
\h \R_m^{\max} (\sH) = \sup_{z^m_1} \h \R_{m} (\sH).
\]

\begin{lemma}
\label{lem:smooth}
Let $g$ be the smoothed margin loss from~\citep[Section 5.1]{SrebroSridharanTewari2010},
with its second moment bounded by $\pi^2/4\rho^2$.
Then, the following holds:
\[
\A_m(\sG) \leq \frac{16 \pi^2 m}{\rho^2} ( \h \R^{\max}_m)^2 (\sH)
 \left(2 \log^{3/2} \frac{m}{\h \R_m^{\max} (\sH)} - \log^{3/2} \frac{2\pi m}{ \rho\h \R_m^{\max} (\sH)}\right)^2.
\]
\end{lemma}
Combining Lemma~\ref{lem:smooth} with Corollary~\ref{cor:rad} yields the
following bound, which is a generalization of \cite[Theorem 5]{SrebroSridharanTewari2010} holding for all $\alpha \in (1, 2]$. Furthermore, our constants are more favorable.

\begin{corollary}
\label{cor:smooth_alpha}
For any $\delta > 0$, with probability at least $1 - \delta$, the following inequality holds for all $\alpha \in (0, 1]$ and all $h \in \sH$:
\[
R(h) - \h
  R^\rho_{S}(h) \leq  32 \sqrt{2} \sqrt[\alpha]{  \h R^\rho_{S}(h)} \, \beta_m^{1-\frac{1}{\alpha}} 
+ \ 2 (32)^{\frac{\alpha}{\alpha-1}} \beta_m,
\]
where 
\[
\beta_m = \frac{16 \pi^2}{\rho^2}  (\h \R^{\max}_m)^2 (\sH)
 \left[2 \log^{3/2} \frac{m}{\h \R_m^{\max} (\sH)} -   \log^{3/2} \frac{2\pi m}{ \rho\h \R_m^{\max} (\sH)}\right]^2 + \frac{\log \log m + \log \frac{16}{\delta}}{m}.
\]
\end{corollary}

\section{Generalization bounds for unbounded loss functions}
\label{sec:unbound}

Standard generalization bounds hold for bounded loss functions.
For the more general and more realistic 
case of unbounded loss functions, a number
of different results have been presented in the past, under
different assumption on the family of functions. 
This includes
learning bounds assuming the
existence of an \emph{envelope}, that is a single non-negative function with
a finite expectation lying above the absolute value of the loss of
every function in the hypothesis set
\citep{Dudley84,Pollard84,Dudley87,Pollard89,Haussler92},
or 
an assumption similar to Hoeffding's inequality based on the
expectation of a hyperbolic function, a quantity similar to the
moment-generating function \citep{MeirZhang2003},
or the weaker assumption that the $\alpha$th-moment of the loss
is bounded for some value of $\alpha > 1$ 
\citep{vapnik98,vapnik06,CortesGreenbergMohri2019}. Here,
we will also adopt this latter assumption and 
present distribution-dependent learning bounds for unbounded losses
that improve upon the previous bounds of 
\cite{CortesGreenbergMohri2019}.
To do so, we will leverage 
the relative deviation margin bounds given 
in the previous sections, which hold 
for any $\alpha \leq 2$.

Let $L$ be an unbounded loss function and $L(h, z)$ denote the loss of hypothesis $h$ for sample $z$. Let $\L_{\alpha}(h) = \E_{z \sim D}[L(h, z)^{\alpha}]$ be the $\alpha^\text{th}$-moment of the loss function $\L$, which is assumed finite for all $h \in \sH$. In what follows, we will use the shorthand $\Pr [L(h, z) > t]$
instead of $\Pr_{z \sim D}[L(h, z) > t]$, and similarly $\h \Pr [L(h, z) > t]$ instead of $\Pr_{z \sim \h D}[L(h, z) > t]$.

\begin{theorem}
\label{th:unbound} Fix $\rho \geq 0$. Let $1 < \alpha \leq 2$, $0 < \epsilon \leq 1$, and $0 < \tau^{\frac{\alpha - 1}{\alpha}} < \epsilon^{\frac{\alpha}{\alpha - 1}}$. For any loss function $L$ (not necessarily
bounded) and hypothesis set $\sH$ such that $\L_{\alpha}(h) < +\infty$ for all $h \in \sH$, 
\begin{multline*}
\Pr \left[\sup_{h \in H} \L(h) - \h \L_{S}(h)\,  > {\Gamma}_{\tau}(\alpha, \epsilon)\, \epsilon {\sqrt[\alpha]{\L_{\alpha}(h) + \tau}}+ \rho
\right] \\
\leq \Pr \left[\sup_{h \in \sH, t \in \Rset} 
\frac{\Pr[L(h, z) > t] - \h \Pr[L(h, z) > t - \rho]}{{\sqrt[\alpha]{\Pr[L(h, z) > t] + \tau}}}
 > \epsilon\right],
\end{multline*}
where ${\Gamma}_{\tau}(\alpha, \epsilon) = \frac{\alpha - 1}{\alpha} (1 + \tau)^{\frac{1}{\alpha}} + \frac{1}{\alpha} \left( \frac{\alpha}{\alpha - 1} \right)^{\alpha - 1} (1 + \left(\frac{\alpha -
1}{\alpha}\right)^{\alpha} \tau^{\frac{1}{\alpha}})^{\frac{1}{\alpha}} \left[ 1 + \frac{\log (1/\epsilon)}{\left( \frac{\alpha}{\alpha - 1} \right)^{\alpha - 1}} \right]^{\frac{\alpha - 1}{\alpha}}$
\ignore{ ${\Gamma}_{\tau}(\beta, \epsilon) = \frac{1}{\beta} + \left( \frac{\beta - 1}{\beta} \right) \beta^{\frac{1}{\beta - 1}} \left[1 + \frac{\log(1/\epsilon)}{\beta^{\frac{1}{\beta - 1}}}
\right]^{\frac{1}{\beta}}$, with $\frac{1}{\alpha} + \frac{1}{\beta} = 1$}.
\end{theorem}
The proof is provided in Appendix~\ref{app:ubound}.
The above theorem can be used in conjunction with our relative deviation margin bounds to obtain strong guarantees for unbounded loss functions and we illustrate it with our $\ell_\infty$ based bounds. Similar techniques can be used to obtain peeling-based Rademacher complexity bounds. Combining Theorems~\ref{th:unbound} and \eqref{th:relative} yields the following corollary.
\begin{corollary}
\label{cor:main1vc} Fix $\rho \geq 0$.  Let $\epsilon < 1$, $1 < \alpha \le 2$.
and hypothesis set $\sH$ such that $\L_{\alpha}(h) < +\infty$ for all $h \in \sH$,
\begin{align*}
  \L(h) - \h \L_{S}(h) 
& \leq \gamma
  \sqrt[\alpha]{\L_{\alpha}(h)} 
  \sqrt{\frac{\log \E[\cN_\infty(\L(\sH), \tfrac{\rho}{2}, x_1^{2m})] + \log \frac{1}{\delta}}{m^{\frac{2 (\alpha - 1)}{\alpha}}}} + \rho,
\end{align*}
where $\gamma = \Gamma_0\left(\alpha,   \sqrt{\frac{\log \E[\cN_\infty(\L(\sH), \tfrac{\rho}{2}, x_1^{2m})] + \log \frac{1}{\delta}}{m^{\frac{2 (\alpha - 1)}{\alpha}}}}
\right) = \mathcal{O}(\log m)$.
\end{corollary}

The upper bound in the above corollary has two terms. The first term is 
based on the covering number and 
decreases with $\rho$ while the second term increases with $\rho$. 
One can choose a suitable value of $\rho$ that minimizes the sum to obtain favorable bounds.\footnote{This requires that the bound holds uniformly for all $\rho$, which can be shown with an additional $\log \log \frac{1}{\rho}$ term (See Corollary~\ref{cor:ubound_all_rho})} Furthermore, the above bound depends on the covering number as opposed to the result of \cite{CortesGreenbergMohri2019}, which depends on the number of dichotomies generated by the hypothesis set. Hence, the above bound is \emph{optimistic} 
and in general is more favorable than the previous known bounds of \cite{CortesGreenbergMohri2019}. We note that instead of using the $\ell_\infty$ based bounds, one can use the Rademacher complexity bounds to obtain better results.
\section{Applications}
\label{sec:applications}
\label{sec:H-applications}

In this section, we briefly highlight some applications of our learning bounds: 
both our covering number and Rademacher complexity margin bounds can be
used to derive finer margin-based guarantees for several commonly used
hypothesis sets. Below we briefly illustrate these applications.

\text{Linear hypothesis sets}: let $\sH$ be the family of liner hypotheses defined by 
\[
\sH = \set{\bx \mapsto \bw \cdot \bx\colon \| \bw \|_2 \leq 1, \bx \in \Rset^n, \| \bx \|_2 \leq R}.
\]
Then, the following upper bound holds for the fat-shattering dimension of $\sH$ \citep{BartlettShaweTaylor1998}: $\fat_\rho(\sH) \leq (R/\rho)^2$.
Plugging in this upper bound in the bound of Corollary~\ref{cor:3}
yields the following:
\begin{equation}
\label{bound:NSVM}
R(h) \leq \h R^\rho_{S}(h) + 2\sqrt{\h R^\rho_{S}(h) \, \beta_m}  
+ \beta_m,
\end{equation}
with $\beta_m = \wt O\left(\frac{(R/\rho)^2} {m} \right)$. In comparison, the best existing margin bound for SVM by \citep[Theorem~1.7]{BartlettShaweTaylor1998} is
\begin{equation}
\label{bound:OSVM}
R(h) \leq \h R^\rho_{S}(h) + c' \sqrt{\beta'_m},
\end{equation}
where $c'$ is some universal constant
and where $\beta'_m = \wt O\left(\frac{(R/\rho)^2} {m} \right)$. 
The margin bound \eqref{bound:NSVM} is thus more favorable than \eqref{bound:OSVM}.

\text{Ensembles of predictors in base hypothesis set $\sH$}: 
let $d$ be the VC-dimension of $\sH$ and consider the family 
of ensembles $\sF = \set{x \mapsto
    \sum_{k = 1}^p w_k h_k(x)\colon h_k \in \sH, w_k \geq 0, \sum_{k =
    1}^p w_k = 1}$. Then, the following upper bound on the fat-shattering dimension holds \citep{BartlettShaweTaylor1998}:
$ \fat_\rho(\sF) \leq c (d/\rho)^2 \log(1/\rho)$, for some universal constant $c$.
Plugging in this upper bound in the bound of Corollary~\ref{cor:3}
yields the following:
\begin{equation}
\label{bound:NEnsembles}
R(h) \leq \h R^\rho_{S}(h) + 2\sqrt{\h R^\rho_{S}(h) \, \beta_m}  
+ \beta_m,
\end{equation}
with $\beta_m = \wt O\left(\frac{(d/\rho)^2}{m} \right)$. In comparison, the best existing margin bound for ensembles such as AdaBoost in terms of the VC-dimension of the base hypothesis given by \cite{SchapireFreundBartlettLee1997} is:
\begin{equation}
\label{bound:OEnsembles}
R(h) \leq \h R^\rho_{S}(h) + c' \sqrt{\beta'_m},
\end{equation}
where $c'$ is some universal constant
and where $\beta'_m = \wt O\left(\frac{(d/\rho)^2}{m} \right)$. The margin bound in \eqref{bound:NEnsembles} is thus more favorable than \eqref{bound:OEnsembles}.

\text{Feed-forward neural networks of depth $d$}: let $\sH_0 = \set{\bx \mapsto \bx_i \colon i \in \{0,1,\ldots n\}, \bx  \in [-1,1]^n} \cup \set{0, 1}$ and 
\[
\sH_i = \set{\sigma\left(\sum_{h \in \cup_{j < i} \sH_j} \bw \cdot h \right): \| \bw \|_1 \leq R }
\]
for $i \in [d]$, where
$\sigma$ is a $\mu$-Lipschitz activation function.
Then, the following upper bound holds for the fat-shattering dimension of $\sH$ \citep{BartlettShaweTaylor1998}:
$\fat_\rho(\sH_d) \leq  \frac{c^{d^2} (R\mu)^{d(d+1)}}{\rho^{2d}} \log n$.
Plugging in this upper bound in the bound of Corollary~\ref{cor:3} gives the following:
\begin{equation}
\label{bound:NNN}
R(h) \leq \h R^\rho_{S}(h) + 2\sqrt{\h R^\rho_{S}(h) \, \beta_m}  
+ \beta_m,
\end{equation}
with 
$\beta_m = \wt O\left( \frac{ c^{d^2} (R\mu)^{d(d+1)}/\rho^{2d} } {m} \right)$. In comparison, the best existing margin bound for neural networks by~\citep[Theorem 1.5 , Theorem 1.11]{BartlettShaweTaylor1998} is
\begin{equation}
\label{bound:ONN}
R(h) \leq \h R^\rho_{S}(h) + c' \sqrt{\beta'_m},
\end{equation}
where $c'$ is some universal constant
and where $\beta'_m = \wt O\left( \frac{ c^{d^2} (R\mu)^{d(d+1)}/\rho^{2d} } {m} \right)$. The margin bound in~\eqref{bound:NNN} is thus more favorable than~\eqref{bound:ONN}.
The Rademacher complexity bounds of Corollary~\ref{cor:smooth_alpha} can also be used to provide generalization bounds for neural networks.
For a matrix $\bW$, let $\|\bW\|_{p, q}$ denote the matrix $p,q$ norm and $\|\bW\|_2$ denote the spectral norm. 
Let $\sH_0 = \{\bx : \|\bx\|_2 \leq 1, \bx \in \bR^n\} $ and $\sH_i = \{\sigma (\bW \cdot h) : h \in \sH_{i-1}, \|\bW\|_2 \leq R,  \|\bW^T\|_{2,1} \leq R_{2,1}\|\bW\|_2  )\}$. Then, by \citep{BartlettFosterTelgarsky2017}, the following
upper bound holds:
\[
\h \R_m^{\max} (\sH)  \leq \wt O\left( \frac{d^{3/2} R R_{2,1}}{\rho^d \sqrt{m}} \cdot (R L)^{d} \right).
\]
Plugging in this upper bound in the bound of Corollary~\ref{cor:smooth_alpha} leads to the following:
\begin{equation}
\label{bound:NNN2}
R(h) \leq \h R^\rho_{S}(h) + 2\sqrt{\h R^\rho_{S}(h) \, \beta_m}  
+ \beta_m,
\end{equation}
where $\beta_m = \wt O\left( \frac{d^{3} R^2 R^2_{2,1}}{\rho^{2d} m} \cdot (R L)^{2d} \right)$. In comparison, the best existing neural network bounds by \citet[Theorem 1.1]{BartlettFosterTelgarsky2017} is
\begin{equation}
\label{bound:ONN2}
R(h) \leq \h R^\rho_{S}(h) + c' \sqrt{\beta'_m},
\end{equation}
where $c'$ is a universal constant and $\beta'_m$ is the empirical Rademacher complexity.
The margin bound \eqref{bound:NNN2} has the benefit of a more favorable dependency on 
the empirical margin loss than \eqref{bound:ONN2}, which can be significant when 
that empirical term is small. On other hand, the empirical Rademacher complexity
of \eqref{bound:ONN2} is more favorable than its counterpart in \eqref{bound:NNN2}.

In Appendix~\ref{app:applications}, we further 
discuss other potential applications of 
our learning guarantees.

\section{Conclusion}

We presented a series of general relative deviation margin bounds.
These are tighter margin bounds that can serve as useful tools to
derive guarantees for a variety of hypothesis sets and in a variety of applications. In particular, these bounds could help derive better
margin-based learning bounds for different families of neural networks,
which has been the topic of several recent research publications.

\ignore{
\section{Acknowledgments}
The work of Mehryar Mohri was partly supported by NSF CCF-1535987, NSF
IIS-1618662, and a Google Research Award.
}

\conf{
\newpage
\section*{Broader impact}

This paper presents a series of theoretical results unlikely to 
admit any immediate or practical relevance to social questions.
Nevertheless, the bounds presented can form a powerful tool for 
the design and analysis of learning algorithms with potentially
pivotal impact.
}

\bibliography{uboundf}

\newpage
\appendix

\conf{
\begin{center}
{\Large Appendix: Relative Deviation Margin Bounds}
\end{center}
}

\section{Symmetrization}
\label{app:symmetrization}

We use the following lemmas from \cite{CortesGreenbergMohri2019} in our
proofs.
\begin{lemma}[\cite{CortesGreenbergMohri2019}]
  \label{lem:monotone} 
  Fix $\eta > 0$ and $\alpha$ with $1 < \alpha \leq 2$. Let
  $f\colon (0, +\infty) \times (0, +\infty) \to \mathbb{R}$ be the
  function defined by
  $f\colon (x, y) \mapsto \frac{x - y}{\sqrt[\alpha]{x + y +
      \eta}}$. Then, $f$ is a strictly increasing function of $x$ and
  a strictly decreasing function of $y$.
\end{lemma}

\begin{lemma}[\cite{GreenbergMohri2013}]
\label{lem:binomial}
Let $X$ be a random variable distributed according to the binomial
distribution $B(m, p)$ with $m$ a positive integer (the number of
trials) and $p > \frac{1}{m}$ (the probability of success of each
trial). Then, the following inequality holds:
\begin{equation}
\label{eq:main} \Pr\left[X \geq \E[X]\right] > \frac{1}{4},
\end{equation}
and, if instead of requiring $p >\frac{1}{m}$ we require $p < 1 - \frac{1}{m}$, then
\begin{equation}
\label{lemma:binomial2} \Pr\left[X \leq \E[X] \right] > \frac{1}{4},
\end{equation}
where in both cases $\E[X] = mp$.
\end{lemma}

The following symmetrization lemma in terms of empirical margin loss is
proven using the previous lemmas.

\begin{replemma}{lem:four}
Fix $\rho \geq 0$ and $1 < \alpha \leq 2$ and assume that
$m \e^{\frac{\alpha}{\alpha - 1}} > 1$. Then, for any any
$\e, \tau > 0$, the following inequality holds:
\begin{equation*}
  \Pr_{S \sim \sD^{m}} \left[\sup_{h \in \sH} \frac{R(h) - \h
      R^\rho_{S}(h)}{\sqrt[\alpha]{R(h) + \tau}} > \e \right] 
  \leq 4 \Pr_{S, S^{\prime} \sim \sD^{m}} \left[\sup_{h \in \sH} \frac{ \h
      R_{S^{\prime}}(h) - \h R^\rho_{S}(h) }{\sqrt[\alpha]{\frac{1}{2} [ \h R_{S^{\prime}}(h) + \h R^\rho_{S}(h) +\frac{1}{m}] } } > \e \right].
\end{equation*}
\end{replemma}
\begin{proof}
We will use the function $F$ defined over $(0, +\infty) \times (0, +\infty)$ by $F\colon (x, y) \mapsto \frac{x - y}{\sqrt[\alpha]{\frac{1}{2}[x
+ y + \frac{1}{m}]}}$.

Fix $S, S' \in \sZ^m$. We first show that the following implication holds for any $h \in \sH$:
\begin{equation}
\label{eq:implication} \left( \frac{R(h) - \h R^\rho_{S}(h)}{\sqrt[\alpha]{R(h) + \tau}} > \e \right) \wedge \left( \h R_{S^{\prime}}(h) > R(h) \right) \Rightarrow F(\h R_{S^{\prime}}(h),
\h R^\rho_{S}(h)) > \e.
\end{equation}
The first condition can be equivalently rewritten as $\h R^\rho_{S}(h) < R(h)
- \e \sqrt[\alpha]{(R(h) + \tau)}$, which implies
\begin{equation}
\label{eq:h0}
\h R^\rho_{S}(h) < R(h) - \e \sqrt[\alpha]{R(h)}\\
\qquad \wedge \qquad \e^{\frac{\alpha}{\alpha - 1}} < R(h),
\end{equation}
since $\h R^\rho_{S}(h) \ge 0$. Assume that the antecedent of the implication (\ref{eq:implication}) holds for $h \in \sH$. Then, in view of the monotonicity properties of function $F$
(Lemma~\ref{lem:monotone}), we can write:
\begin{align*}
F(\h R_{S^{\prime}}(h), \h R^\rho_{S}(h)) 
& \geq F(R(h), R(h) - \e \sqrt[\alpha]{R(h)}) & \text{($\h R_{S^{\prime}}(h) > R(h)$ and 1st ineq. of (\ref{eq:h0}))}\\
& = \frac{R(h) - (R(h) - \e R(h)^{\frac{1}{\alpha}}}{\sqrt[\alpha]{\frac{1}{2}[2 R(h) - \e R(h)^{\frac{1}{\alpha}} + \frac{1}{m}]}} \\
& \geq \frac{\e R(h)^{\frac{1}{\alpha}}}{\sqrt{\frac{1}{2}[2 R(h) -
\e^{\frac{\alpha}{\alpha - 1}} + \frac{1}{m}]}} & \text{(second ineq. of (\ref{eq:h0}))}\\
& > \frac{\e R(h)^{\frac{1}{\alpha}}}{\sqrt[\alpha]{\frac{1}{2}[2 R(h) ]}} =
  \e, & \text{($m \e^{\frac{\alpha}{\alpha - 1}} > 1$)}
\end{align*}
which proves (\ref{eq:implication}). 

Now, by definition of the supremum, for any $\eta > 0$, there exists $h_S \in \sH$ such that
\begin{equation}
\label{eq:sup} \sup_{h \in \sH} \frac{R(h) - \h R^\rho_{S}(h)}{\sqrt[\alpha]{R(h) + \tau}} - \frac{R(h_S) - \h R^\rho_{S}(h_S)}{\sqrt[\alpha]{R(h_S) + \tau}}  \le \eta.
\end{equation}
Using the definition of $h_S$ and the 
implication (\ref{eq:implication}), we can write
\begin{align*}
& \Pr_{S, S^{\prime} \sim \sD^{m}} \left[\sup_{h \in \sH} \frac{ \h R_{S^{\prime}}(h) - \h R^\rho_{S}(h) }{ \sqrt[\alpha]{\frac{1}{2} [\h R^\rho_{S}(h) + \h R_{S^{\prime}}(h) + \frac{1}{m}]} } > \e \right] \\
& \geq \Pr_{S, S^{\prime} \sim \sD^{m}} \left[\frac{ \h R_{S^{\prime}}(h_S) - \h R^\rho_{S}(h_S) }{ \sqrt[\alpha]{\frac{1}{2} [\h R^\rho_{S}(h_S) + \h R_{S^{\prime}}(h_S) + \frac{1}{m}]}
} > \e \right] & \text{(def. of $\sup$)}\\
& \geq \Pr_{S, S^{\prime} \sim \sD^{m}} \left[ \left( \frac{R(h_S) - \h R^\rho_{S}(h_S)}{\sqrt[\alpha]{R(h_S)+\tau}} > \e \right) \wedge \left( \h R_{S^{\prime}}(h_S) > R(h_S) \right) \right] & (\text{implication (\ref{eq:implication})})\\
& = \E_{S, S' \sim \sD^{m}} \left[1_{\frac{R(h_S) - \h
      R^\rho_{S}(h_S)}{\sqrt[\alpha]{R(h_S)+\tau}} > \e}
      1_{\h R_{S^{\prime}}(h_S) > R(h_S ) } \right] & (\text{def. of expectation})\\
& = \E_{S \sim \sD^{m}} \left[1_{\frac{R(h_S) - \h
      R^\rho_{S}(h_S)}{\sqrt[\alpha]{R(h_S)+\tau}} > \e}
      \Pr_{S^{\prime} \sim \sD^{m}} \left[ \h
      R_{S^{\prime}}(h_S) > R(h_S ) \right] \right]. &
                                                       \text{(linearity
                                                       of expectation)}
\end{align*}
Now, observe that,
if $R(h_S) \leq \e^{\frac{\alpha}{\alpha - 1}}$, then the following inequalities hold:
\begin{align}
\frac{R(h_S) - \h R^\rho_{S}(h_S)}{\sqrt[\alpha]{R(h_S) + \tau}}
\leq \frac{R(h_S)}{\sqrt[\alpha]{R(h_S)}} =  R(h_S)^{\frac{\alpha
  - 1}{\alpha}} \leq \e.
\end{align}
In light of that, we can write
\begin{align*}
& \Pr_{S, S^{\prime} \sim \sD^{m}} \left[\sup_{h \in \sH} \frac{ \h R_{S^{\prime}}(h) - \h R^\rho_{S}(h) }{ \sqrt[\alpha]{\frac{1}{2} [\h R^\rho_{S}(h) + \h R_{S^{\prime}}(h) + \frac{1}{m}]} } > \e \right] \\
& \geq \E_{S \sim \sD^{m}} \left[1_{\frac{R(h_S) - \h
      R^\rho_{S}(h_S)}{\sqrt[\alpha]{R(h_S) + \tau}} > \e} 1_{R(h_S) > \e^{\frac{\alpha}{\alpha - 1}}}
      \Pr_{S^{\prime} \sim \sD^{m}} \left[ \h
      R_{S^{\prime}}(h_S) > R(h_S ) \right] \right]\\
& \geq \frac{1}{4} \E_{S \sim \sD^{m}} \left[1_{\frac{R(h_S) - \h
      R^\rho_{S}(h_S)}{\sqrt[\alpha]{R(h_S) + \tau}} > \e} \right] &
                                                                 (\text{$\e^{\frac{\alpha}{\alpha - 1}}
                                                                 >
                                                                 \tfrac{1}{m}$
                                                                 and
                                                                 Lemma~\ref{lem:binomial}})\\
& \geq \frac{1}{4} \E_{S \sim \sD^{m}} \left[1_{\sup_{h \in \sH} \frac{R(h) - \h
      R^\rho_{S}(h)}{\sqrt[\alpha]{R(h) + \tau}} > \e + \eta} \right] &
                                                                 (\text{def. of
                                                                 $h_S$})\\
& = \frac{1}{4} \Pr_{S \sim \sD^m} \left[\sup_{h \in \sH} \frac{R(h) - \h
      R^\rho_{S}(h)}{\sqrt[\alpha]{R(h) + \tau}} > \e + \eta \right]. &
                                                                 (\text{def. of
                                                                 expectation})
\end{align*}
Now, since this inequality holds for all $\eta > 0$, we can take the
limit $\eta \to 0$ and use the right-continuity of the cumulative
distribution to obtain
\begin{equation*}
\Pr_{S, S^{\prime} \sim \sD^{m}} \left[\sup_{h \in \sH} \frac{ \h R_{S^{\prime}}(h) - \h R^\rho_{S}(h) }{ \sqrt[\alpha]{\frac{1}{2} [\h R^\rho_{S}(h) + \h R_{S^{\prime}}(h) + \frac{1}{m}]} } > \e
\right] \geq \frac{1}{4} \Pr_{S \sim \sD^{m}} \left[ \sup_{h \in \sH} \frac{R(h) - \h R^\rho_{S}(h)}{\sqrt[\alpha]{R(h) + \tau}} > \e \right],
\end{equation*}
which completes the proof.
\end{proof}

\begin{replemma}{lem:cover}
Fix $\rho \geq 0$ and $1 < \alpha \leq 2$. Then, the
following inequality holds:
\begin{equation*}
\Pr_{S, S^{\prime} \sim \sD^{m}} \left[\sup_{h \in \sH} \frac{ \h
R_{S^{\prime}}(h) - \h R^\rho_{S}(h) }{ \sqrt[\alpha]{\frac{1}{2} [ \h
  R_{S^{\prime}}(h) + \h R^\rho_{S}(h) +\frac{1}{m}] } } > \e
\right] \leq
\Pr_{S, S^{\prime} \sim \sD^{m}} \left[\sup_{g \in \sG} \frac{ \h
R_{S^{\prime}}(g) - \h R_{S}(g) }{ \sqrt[\alpha]{\frac{1}{2} [ \h R_{S^{\prime}}(g) + \h R_{S}(g) +\frac{1}{m}] } } > \e \right].
\end{equation*}
Further when $g(z)= 1_{yh(x) < \rho/2}$, then 
\begin{equation*}
\Pr_{S, S^{\prime} \sim \sD^{m}} \left[\sup_{h \in \sH} \frac{ \h
R_{S^{\prime}}(h) - \h R^\rho_{S}(h) }{ \sqrt[\alpha]{\frac{1}{2} [ \h
  R_{S^{\prime}}(h) + \h R^\rho_{S}(h) +\frac{1}{m}] } } > \e
\right] \leq
\Pr_{S, S^{\prime} \sim \sD^{m}} \left[\sup_{h \in \sC(\sH_\rho,
    \frac{\rho}{2}, S \cup S')} \frac{ \h
R^{\frac{\rho}{2}}_{S^{\prime}}(h) - \h R^{\frac{\rho}{2}}_{S}(h) }{ \sqrt[\alpha]{\frac{1}{2} [ \h R^{\frac{\rho}{2}}_{S^{\prime}}(h) + \h R^{\frac{\rho}{2}}_{S}(h) +\frac{1}{m}] } } > \e \right].
\end{equation*}
\end{replemma}
\begin{proof}
For the first part of the lemma, note that for any given $h$ and the corresponding $g$, and sample $z \in S \cup S'$, using inequalities 
\[
1_{yh(x) <0} \leq g(z) \leq 1_{yh(x) < \rho}.
\]
and taking expectations yields for any sample $S$:
\[
\h R_{S}(h) \leq R_{S}(g) \leq \h R^\rho_{S}(h) .
\]
The result then follows by Lemma~\ref{lem:monotone}.

For the second part of the lemma, observe that restricting the output of $h \in \sH$
to be in $[-\rho, \rho]$ does not change its binary
or margin-loss:  $1_{y h(x) < \rho} = 1_{y h_\rho(x) < \rho}$
and $1_{y h(x) \leq 0} = 1_{y h_\rho(x) \leq 0}$. Thus,
we can write
\[
\Pr_{S, S^{\prime} \sim \sD^{m}} \left[\sup_{h \in \sH} \frac{ \h
R_{S^{\prime}}(h) - \h R^\rho_{S}(h) }{ \sqrt[\alpha]{\frac{1}{2} [ \h R_{S^{\prime}}(h) + \h R^\rho_{S}(h) +\frac{1}{m}] } } > \e \right]
= \Pr_{S, S^{\prime} \sim \sD^{m}} \left[\sup_{h \in \sH_\rho} \frac{ \h
R_{S^{\prime}}(h) - \h R^\rho_{S}(h) }{ \sqrt[\alpha]{\frac{1}{2} [ \h R_{S^{\prime}}(h) + \h R^\rho_{S}(h) +\frac{1}{m}] } } > \e \right].
\]
Now, by definition of $\sC(\sH_\rho, \frac{\rho}{2}, S \cup S')$,
for any $h \in \sH_\rho$ there exists $g \in \sC(\sH_\rho,
\frac{\rho}{2}, S \cup S')$ such that for any  $x \in S \cup S'$,
\[
|g(x) - h(x)| \leq \frac{\rho}{2}.
\]
Thus, for any $y \in \set{-1, +1}$ and $x \in S \cup S'$, we have
$|yg(x) - yh(x)| \leq \frac{\rho}{2}$, which implies:
\[
1_{y h(x) \leq 0} \leq 1_{y g(x) \leq \frac{\rho}{2}} \leq 1_{y h(x) \leq \rho}.
\]
Hence, we have
$\h R_{S^{\prime}}(h) \leq \h R^{\frac{\rho}{2}}_{S^{\prime}}(g)$ and
$\h R^\rho_{S}(h)  \geq \h R^{\frac{\rho}{2}}_{S}(g)$ and,
by the monotonicity properties of Lemma~\ref{lem:monotone}:
\[
\frac{ \h
R_{S^{\prime}}(h) - \h R^\rho_{S}(h) }{ \sqrt[\alpha]{\frac{1}{2} [ \h R_{S^{\prime}}(h) + \h R^\rho_{S}(h) +\frac{1}{m}] } }
\leq 
 \frac{ \h
R^{\frac{\rho}{2}}_{S^{\prime}}(g) - \h R^{\frac{\rho}{2}}_{S}(g) }{ \sqrt[\alpha]{\frac{1}{2} [ \h R^{\frac{\rho}{2}}_{S^{\prime}}(g) + \h R^{\frac{\rho}{2}}_{S}(g) +\frac{1}{m}] } }.
\]
Taking the supremum over both sides yields the result.
\end{proof}

\newpage
\section{Relative deviation margin bounds -- Covering numbers}
\label{app:covering}

\begin{reptheorem}{th:relative} 
Fix $\rho \geq 0$ and $1 < \alpha \leq 2$. Then,
  for any hypothesis set $\sH$ of functions mapping from $\sX$ to
  $\Rset$ and any $\tau > 0$, the following inequality holds:
\begin{align*}
& \Pr_{S \sim \sD^{m}} \left[\sup_{h \in \sH} \frac{R(h) - \h
  R^\rho_{S}(h)}{\sqrt[\alpha]{R(h) + \tau}} > \e \right] 
\leq 4 \, \E_{x_1^{2m} \sim \sD^{2m}}[\cN_\infty(\sH_\rho, \tfrac{\rho}{2}, x_1^{2m})] \ 
\exp \left[ \frac{-m^{\frac{2 (\alpha - 1)}{\alpha}} \e^2}{2^{\frac{\alpha + 2}{\alpha}}} \right].
\end{align*}
\end{reptheorem}
\begin{proof}
Consider first the case where $m \e^{\frac{\alpha}{\alpha - 1}}
\leq 1$. The bound then holds trivially since we have:
\[
4 \exp \left( \frac{-m^{\frac{2 (\alpha - 1)}{\alpha}}
    \e^2}{2^{\frac{\alpha + 2}{\alpha}}} \right)
\geq 4 \exp \left( \frac{-1}{2^{\frac{\alpha + 2}{\alpha}}} \right) > 1.
\]
On the other hand, when $m \e^{\frac{\alpha}{\alpha - 1}} > 1$, 
by Lemmas~\ref{lem:four} and~\ref{lem:cover} we can write:
\[
\Pr_{S \sim \sD^{m}} \left[\sup_{h \in \sH} \frac{R(h) - \h R^\rho_{S}(h)}{\sqrt[\alpha]{R(h) + \tau}} > \e \right]
\leq 4 
\Pr_{S, S^{\prime} \sim \sD^{m}} \left[\sup_{h \in \sC(\sH_\rho,
    \frac{\rho}{2}, S \cup S') } \frac{ \h
R^{\frac{\rho}{2}}_{S^{\prime}}(h) - \h R^{\frac{\rho}{2}}_{S}(h) }{ \sqrt[\alpha]{\frac{1}{2} [ \h R^{\frac{\rho}{2}}_{S^{\prime}}(h) + \h R^{\frac{\rho}{2}}_{S}(h) +\frac{1}{m}] } } > \e \right].
\]
To upper bound the probability that the symmetrized expression is
larger than $\e$, we begin by introducing a vector of Rademacher
random variables
$\sigma = (\sigma_{1}, \sigma_{2}, \ldots, \sigma_{m})$, where
$\sigma_i$s are independent identically distributed random variables
each equally likely to take the value $+1$ or $-1$. Let
$x_1, x_2,\ldots x_m$ be samples in $S$ and
$x_{m+1}, x_{m+2},\ldots x_{2m}$ be samples in $S'$. Using the
shorthands $z = (x, y)$, $g(z) = 1_{yh(x) \leq \frac{\rho}{2}}$, and
$\sG(x_1^{2m}) = \sC(\sH_\rho, \frac{\rho}{2}, S \cup S')$, we can then
write the above quantity as
\begin{align*}
& \Pr_{S, S^{\prime} \sim \sD^{m}} \left[\sup_{h \in \sC(\sH_\rho, \frac{\rho}{2}, S \cup S')} \frac{ \h
R^{\frac{\rho}{2}}_{S^{\prime}}(h) - \h R^{\frac{\rho}{2}}_{S}(h) }{ \sqrt[\alpha]{\frac{1}{2} [ \h R^{\frac{\rho}{2}}_{S^{\prime}}(h) + \h R^{\frac{\rho}{2}}_{S}(h) +\frac{1}{m}] } } > \e \right]\\
& = \Pr_{z_{1}^{2m} \sim \sD^{2m}} \left[\sup_{g \in \sG(x^{2m})}  \frac{ \frac{1}{m}\sum_{i = 1}^{m} (g(z_{m + i}) - g(z_{i})) }{ \sqrt[\alpha]{\frac{1}{2m} [\sum_{i = 1}^{m}
(g(z_{m + i}) + g(z_{i})) + 1]} } > \e \right]\\
& = \Pr_{z_{1}^{2m} \sim \sD^{2m}, \bsigma} \left[\sup_{g \in \sG(x^{2m})} \frac{ \frac{1}{m} \sum_{i = 1}^{m} \sigma_{i} (g(z_{m + i}) - g(z_{i})) }{ \sqrt[\alpha]{\frac{1}{2m} [\sum_{i = 1}^{m}
(g(z_{m + i}) + g(z_{i})) + 1]} } > \e \right]\\
& = \E_{z_{1}^{2m} \sim \sD^{2m}} \left[ \Pr_{\bsigma} \left[\sup_{g \in \sG(x^{2m})} \frac{\frac{1}{m} \sum_{i = 1}^{m} \sigma_{i} (g(z_{m + i}) - g(z_{i})) }{ \sqrt[\alpha]{\frac{1}{2m} [\sum_{i
= 1}^{m} (g(z_{m + i}) + g(z_{i}))+1]} } > \e \, \bigg| \, z_{1}^{2m} \right]\right].
\end{align*}
Now, for a fixed $z_{1}^{2m}$, we have
$\E_{\bsigma}\left[\frac{\frac{1}{m} \sum_{i = 1}^{m} \sigma_{i}
    (g(z_{m + i}) - g(z_{i})) }{ \sqrt[\alpha]{\frac{1}{2m} [\sum_{i =
        1}^{m} (g(z_{m + i}) + g(z_{i}))+1]} }\right] = 0$, thus, by
Hoeffding's inequality, we can write
\begin{align*}
\Pr_{\bsigma} \left[\frac{\frac{1}{m} \sum_{i = 1}^{m} \sigma_{i} (g(z_{m + i}) - g(z_{i})) }{ \sqrt[\alpha]{\frac{1}{2m} [\sum_{i
= 1}^{m} (g(z_{m + i}) + g(z_{i}))+1]} } > \e \, \bigg| \, z_{1}^{2m} \right] 
& \leq \exp \left(  \frac{-[\sum_{i = 1}^{m} (g(z_{m + i}) + g(z_{i})) + 1]^{\frac{2}{\alpha}}
m^{\frac{2 (\alpha - 1)}{\alpha}}\e^{2} }{2^{\frac{\alpha + 2}{\alpha}} \sum_{i = 1}^{m} (g(z_{m + i}) - g(z_{i}))^{2}} \right) \\
& \leq \exp \left(  \frac{-[\sum_{i = 1}^{m} (g(z_{m + i}) + g(z_{i}))]^{\frac{2}{\alpha}} m^{\frac{2 (\alpha - 1)}{\alpha}}\e^{2} }{2^{\frac{\alpha + 2}{\alpha}} \sum_{i = 1}^{m}
(g(z_{m + i}) - g(z_{i}))^{2}} \right).
\end{align*}
Since the variables $g(z_{i})$, $i \in [1, 2m]$, take values in $\set{0, 1}$, we can write
\begin{align*}
\sum_{i = 1}^{m} (g(z_{m + i}) - g(z_{i}))^{2}
& = \sum_{i = 1}^{m} g(z_{m + i}) + g(z_{i}) -2 g(z_{m + i})g(z_{i})\\
& \leq \sum_{i = 1}^{m} g(z_{m + i}) + g(z_{i})\\
& \leq \sum_{i = 1}^{m} \left[ g(z_{m + i}) + g(z_{i}) \right]^{\frac{2}{\alpha}},
\end{align*}
where the last inequality holds since $\alpha \leq 2$ and since
the sum is either zero or greater than or equal to one. In view of this identity, we can write
\begin{equation*}
\Pr_{\bsigma} \left[ \frac{ \frac{1}{m} \sum_{i = 1}^{m} \sigma_{i} (g(z_{m + i}) - g(z_{i})) }{ \sqrt[\alpha]{\frac{1}{2m} [ \sum_{i = 1}^{m} (g(z_{m + i}) + g(z_{i}))]} } > \e \,
\bigg| \, z_{1}^{2m} \right] \leq \exp \left(  \frac{-m^{\frac{2 (\alpha - 1)}{\alpha}}\e^{2} }{2^{\frac{\alpha + 2}{\alpha}}  } \right).
\end{equation*}
The number of such hypotheses is
$\cN_\infty(\sH_\rho, \tfrac{\rho}{2}, x_1^{2m})$, thus, by the union bound, the following holds:
\begin{equation*}
\Pr_{\bsigma} \left[\sup_{g \in \sG(x^{2m})} \frac{ \sum_{i = 1}^{m} \sigma_{i} (g(z_{m + i}) - g(z_{i})) }{ \sqrt[\alpha]{\frac{1}{2} [\sum_{i = 1}^{m} (g(z_{m + i}) + g(z_{i}))]} } > \e \,
\bigg| \, z_{1}^{2m} \right] \leq \cN_\infty(\sH_\rho, \tfrac{\rho}{2}, x_1^{2m}) \exp \left(  \frac{-m^{\frac{2 (\alpha - 1)}{\alpha}}\e^{2} }{2^{\frac{\alpha + 2}{\alpha}}  } \right).
\end{equation*}
The result follows by taking expectations with respect to $z_{1}^{2m}$
and applying the previous lemmas. 
\end{proof}

\newpage
\section{Relative deviation margin bounds -- Rademacher complexity}
\label{app:rademacher}

The following lemma relates the symmetrized expression of
Lemma~\ref{lem:cover} to a Rademacher average quantity.

\begin{replemma}{lem:to_rad}
Fix $1 < \alpha \leq 2$. Then, the
following inequality holds:
\begin{equation*}
\Pr_{S, S^{\prime} \sim \sD^{m}} \left[\sup_{g \in \sG} \frac{ \h
R_{S^{\prime}}(g) - \h R_{S}(g) }{ \sqrt[\alpha]{\frac{1}{2} [ \h R_{S^{\prime}}(g) + \h R_{S}(g) +\frac{1}{m}] } } > \e \right]
\leq 
2 \Pr_{z_{1}^{m} \sim \sD^{m}, \bsigma} \left[\sup_{g \in \sG} \frac{ \frac{1}{m} \sum_{i = 1}^{m} \sigma_{i} g(z_{i}) }{ \sqrt[\alpha]{\frac{1}{m} [\sum_{i = 1}^{m}
(g(z_{i})) + 1]} } > \frac{\e}{2\sqrt{2}} \right].
\end{equation*}
\end{replemma}
\begin{proof}
To upper bound the probability that the symmetrized expression is
larger than $\e$, we begin by introducing a vector of Rademacher
random variables
$\sigma = (\sigma_{1}, \sigma_{2}, \ldots, \sigma_{m})$, where
$\sigma_i$s are independent identically distributed random variables
each equally likely to take the value $+1$ or $-1$.
Let
$z_1, z_2,\ldots z_m$ be samples in $S$ and
$z_{m+1}, z_{m+2},\ldots z_{2m}$ be samples in $S'$. We can then
write the above quantity as
\begin{align*}
& \Pr_{S, S^{\prime} \sim \sD^{m}} \left[\sup_{g \in \sG} \frac{ \h
R_{S^{\prime}}(g) - \h R_{S}(g) }{ \sqrt[\alpha]{\frac{1}{2} [ \h R_{S^{\prime}}(g) + \h R_{S}(g) +\frac{1}{m}] } } > \e \right]
\\
& = \Pr_{z_{1}^{2m} \sim \sD^{2m}} \left[\sup_{g \in \sG}  \frac{ \frac{1}{m}\sum_{i = 1}^{m} (g(z_{m + i}) - g(z_{i})) }{ \sqrt[\alpha]{\frac{1}{2m} [\sum_{i = 1}^{m}
(g(z_{m + i}) + g(z_{i})) + 1]} } > \e \right]\\
& = \Pr_{z_{1}^{2m} \sim \sD^{2m}, \bsigma} \left[\sup_{g \in \sG} \frac{ \frac{1}{m} \sum_{i = 1}^{m} \sigma_{i} (g(z_{m + i}) - g(z_{i})) }{ \sqrt[\alpha]{\frac{1}{2m} [\sum_{i = 1}^{m}
(g(z_{m + i}) + g(z_{i})) + 1]} } > \e \right].
\end{align*}
If $a +b \geq \e$, then either $a \geq \e/2$ or $b \geq \e/2$, hence
\begin{align*}
&\Pr_{z_{1}^{2m} \sim \sD^{2m}, \bsigma} \left[\sup_{g \in \sG} \frac{ \frac{1}{m} \sum_{i = 1}^{m} \sigma_{i} (g(z_{m + i}) - g(z_{i})) }{ \sqrt[\alpha]{\frac{1}{2m} [\sum_{i = 1}^{m}
(g(z_{m + i}) + g(z_{i})) + 1]} } > \e \right] \\
& \leq \Pr_{z_{1}^{2m} \sim \sD^{2m}, \bsigma} \left[\sup_{g \in \sG} \frac{ \frac{1}{m} \sum_{i = 1}^{m} \sigma_{i} (g(z_{m + i})) }{ \sqrt[\alpha]{\frac{1}{2m} [\sum_{i = 1}^{m}
(g(z_{m + i}) + g(z_{i})) + 1]} } > \frac{\e}{2} \right] \\
& + \Pr_{z_{1}^{2m} \sim \sD^{2m}, \bsigma} \left[\sup_{g \in \sG} \frac{ \frac{1}{m} \sum_{i = 1}^{m} \sigma_{i} ( - g(z_{i})) }{ \sqrt[\alpha]{\frac{1}{2m} [\sum_{i = 1}^{m}
(g(z_{m + i}) + g(z_{i})) + 1]} } > \frac{\e}{2} \right] \\
&= 2 \Pr_{z_{1}^{2m} \sim \sD^{2m}, \bsigma} \left[\sup_{g \in \sG} \frac{ \frac{1}{m} \sum_{i = 1}^{m} \sigma_{i} g(z_{i}) }{ \sqrt[\alpha]{\frac{1}{2m} [\sum_{i = 1}^{m}
(g(z_{m + i}) + g(z_{i})) + 1]} } > \frac{\e}{2} \right]  \\
&\leq  2 \Pr_{z_{1}^{2m} \sim \sD^{2m}, \bsigma} \left[\sup_{g \in \sG} \frac{ \frac{1}{m} \sum_{i = 1}^{m} \sigma_{i} g(z_{i}) }{ \sqrt[\alpha]{\frac{1}{2m} [\sum_{i = 1}^{m}
(g(z_{i})) + 1]} } > \frac{\e}{2} \right]  \\
& \leq 2 \Pr_{z_{1}^{2m} \sim \sD^{2m}, \bsigma} \left[\sup_{g \in \sG} \frac{ \frac{1}{m} \sum_{i = 1}^{m} \sigma_{i} g(z_{i}) }{ \sqrt[\alpha]{\frac{1}{m} [\sum_{i = 1}^{m}
(g(z_{i})) + 1]} } > \frac{\e}{2\sqrt{2}} \right] \\
& = 2 \Pr_{z_{1}^{m} \sim \sD^{m}, \bsigma} \left[\sup_{g \in \sG} \frac{ \frac{1}{m} \sum_{i = 1}^{m} \sigma_{i} g(z_{i}) }{ \sqrt[\alpha]{\frac{1}{m} [\sum_{i = 1}^{m}
(g(z_{i})) + 1]} } > \frac{\e}{2\sqrt{2}} \right],
\end{align*}
where the penultimate inequality follow by observing that if $a/c \geq \e$, then $a/c' \geq \e$, for all $c' \leq c$ and the last inequality follows by observing $\alpha \geq 1$.
\end{proof}

We will use the following bounded difference inequality \citep[Theorem 3.18]{VanHandel2016}, which provide us with
a finer tool that McDiarmid's inequality.

      \begin{lemma}[\citep{VanHandel2016}]
      \label{lem:ramon}
      Let $f(x_1, x_2, \ldots, x_n)$ be a function of $n$ independent samples $x_1, x_2, \ldots x_n$. 
      Let 
      \[
      c_i = \max_{x'_i} f(x_1, x_2, \ldots, x_n)
      - f(x_1, x_2, \ldots, x_{i-1}, x'_i, x_{i+1}, \ldots, x_n).
      \]
      Then,
      \[
      \Pr\left(f(x_1, x_2, \ldots, x_n) \geq \E[f(x_1, x_2, \ldots, x_n)] + \epsilon \right)
      \leq \exp \left( - \frac{\epsilon^2}{4 \sum_i c^2_i} \right).
      \]
      \end{lemma}

Using the above inequality and a peeling argument, we show the following upper bound expressed in terms of Rademacher complexities.
\begin{replemma}{lem:indic}
Fix $1 < \alpha \leq 2 $ and $z_1^m \in \sZ^m$. Then, the following
inequality holds:
\[
\Pr_{\bsigma} \left[\sup_{g \in \sG} \frac{ \frac{1}{m} \sum_{i = 1}^{m} \sigma_{i} g(z_{i}) }{ \sqrt[\alpha]{\frac{1}{m} [\sum_{i = 1}^{m}
(g(z_{i})) + 1]} } > {\e} \mid z^{m}   \right] \leq 2 \sum_{k =
0}^{\lfloor \log_2 m \rfloor}
\exp \left( \frac{m^2\h \R_m^2(\sG_k(z_1^m))}{2^{k+5}}  - \frac{\e^2}{64\frac{2^{k(1-2/\alpha)}}{m^{2-2/\alpha}}} \right)  1_{\e \leq 2
\left(\frac{2^{k}}{m} \right)^{1-1/\alpha}}.
\]
\end{replemma}
\begin{proof}
By definition of $\sG_k$, the following inequality holds:
\[
\sup_{g \in \sG_k(z_1^m)} \frac{\frac{1}{m} \sum_{i = 1}^{m} \sigma_{i} g(z_{i})}{\sqrt[\alpha]{\frac{1}{m} [\sum_{i = 1}^{m}
(g(z_{i})) + 1]}} \leq \frac{\frac{2^{k+1}}{m}}{\sqrt[\alpha]{\frac{1}{m} [\sum_{i = 1}^{m}
(g(z_{i})) + 1]}} \leq \frac{\frac{2^{k+1}}{m}}{\left(\frac{2^k}{m}\right)^{1/\alpha}}.
\]
Thus, for $\e > 2
\left(\frac{2^{k}}{m} \right)^{1-1/\alpha}$, the left-hand side probability is
zero. This leads to the indicator function factor in the right-hand
side of the expression. We now prove the non-indicator part.

By the union bound,
\begin{align*}
\Pr_{\bsigma} \left[\sup_{g \in \sG} \frac{ \frac{1}{m} \sum_{i = 1}^{m} \sigma_{i} g(z_{i}) }{ \sqrt[\alpha]{\frac{1}{m} [\sum_{i = 1}^{m}
(g(z_{i})) + 1]} } > \e \, \Bigg\mid \, z^{m} \right]
& = 
\Pr_{\bsigma} \left[\sup_{k} \sup_{g \in \sG_k(z_1^m)} \frac{ \frac{1}{m} \sum_{i = 1}^{m} \sigma_{i} g(z_{i}) }{ \sqrt[\alpha]{\frac{1}{m} [\sum_{i = 1}^{m}
(g(z_{i})) + 1]} } > \e \, \Bigg\mid \, z^{m} \right] \\
& \leq  
\sum_k \Pr_{\bsigma} \left[ \sup_{g \in \sG_k(z_1^m)} \frac{ \frac{1}{m} \sum_{i = 1}^{m} \sigma_{i} g(z_{i}) }{ \sqrt[\alpha]{\frac{1}{m} [\sum_{i = 1}^{m}
(g(z_{i})) + 1]} } > \e \, \Bigg\mid \, z^{m} \right] \\
& \leq  
\sum_k \Pr_{\bsigma} \left[ \sup_{g \in \sG_k(z_1^m)} \frac{ \frac{1}{m} |\sum_{i = 1}^{m} \sigma_{i} g(z_{i}) | }{ \sqrt[\alpha]{\frac{1}{m} [\sum_{i = 1}^{m}
(g(z_{i})) + 1]} } > \e \, \Bigg\mid \, z^{m} \right] \\
& \stackrel{(a)}{\leq}
\sum_k  \Pr_{\bsigma} \left[ \sup_{g \in \sG_k(z_1^m)} { \frac{1}{m} |\sum_{i = 1}^{m} \sigma_{i} g(z_{i}) | } > \e \sqrt[\alpha]{\frac{2^k}{m}} \, \Bigg\mid \, z^{m} \right] \\
& \stackrel{(b)}{\leq} 
\sum_k 2\Pr_{\bsigma} \left[ \sup_{g \in \sG_k(z_1^m)} { \frac{1}{m} \sum_{i = 1}^{m} \sigma_{i} g(z_{i})  } > \e \sqrt[\alpha]{\frac{2^k}{m}} \, \Bigg\mid \, z^{m} \right],
\end{align*}
where the $(a)$ follows by observing that for all $g\in \sG_k$, $[\sum_{i = 1}^{m}
(g(z_{i})) + 1] \geq 2^{k}/m$ and $(b)$ follows by observing that for a particular $\bsigma$, $\frac{1}{m} \sum_{i = 1}^{m} \sigma_{i} g(z_{i}) < \e \sqrt[\alpha]{\frac{2^k}{m}} $, then for $\bsigma' = -\bsigma$, the value would be 
$\frac{1}{m} \sum_{i = 1}^{m} \sigma'_{i} g(z_{i}) > \e \sqrt[\alpha]{\frac{2^k}{m}} $.
Hence it suffices to bound
\[
\Pr_{\bsigma} \left[ \sup_{g \in \sG_k(z_1^m)} \frac{1}{m} \sum_{i = 1}^{m} \sigma_{i} g(z_{i}) > \e\sqrt[\alpha]{\frac{2^{k}}{m}} \, \Bigg\mid \, z^{m} \right],
\]
for a given $k$. 
We will apply the bounded difference inequality
(\citep[Theorem 3.18]{VanHandel2016}), which is a finer concentration
bound than McDiarmid's inequality in this context, to the random
variable
$\sup_{g \in \sG_k(z_1^m)} \frac{1}{m} \sum_{i = 1}^{m} \sigma_{i}
g(z_{i})$. For any $\bsigma$, let $g_{\bsigma}$ denote the function in
$\sG_k(z_1^m)$ that achieves the supremum. For simplicity, we assume that the supremum can be achieved. The proof can be extended to the case when its not achieved. Then, for any two vectors
of Rademacher variables $\bsigma$ and $\bsigma'$ that differ only in
the $j^{\text{th}}$ coordinate, the difference of suprema can be
bounded as follows:
\begin{align*}
  \frac{1}{m} \sum_{i = 1}^{m} \sigma_{i} g_{\bsigma}(z_{i})
  -  \frac{1}{m} \sum_{i = 1}^{m} \sigma'_{i} g_{\bsigma'}(z_{i})
  & \leq  \frac{1}{m} \sum_{i = 1}^{m} \sigma_{i} g_{\bsigma}(z_{i})
  -  \frac{1}{m} \sum_{i = 1}^{m} \sigma'_{i} g_{\bsigma}(z_{i}) \\
  & = \frac{1}{m} (\sigma_j - \sigma'_j) g_\bsigma(z_j)\\
  & \leq \frac{2g_{\bsigma}(z_j)}{m}.
\end{align*}
The sum of the squares of the changes is therefore bounded by
\[
\frac{4}{m^2} \sum_{i = 1}^m g^2_{\bsigma}(z_i)
\leq \frac{4}{m^2} \sup_{g \in \sG_k(z_1^m)} \sum_{i = 1}^m g^2(z_i) 
\leq \frac{4}{m^2} \sup_{g \in \sG_k(z_1^m)} \sum_{i = 1}^m g(z_i) 
\leq \frac{4}{m^2} m 2^{k + 1} 
= \frac{2^{k + 3}}{m}.
\]
Since
$\E_\bsigma \left[ \sup_{g \in \sG_k(z_1^m)} \frac{1}{m} \sum_{i =
    1}^{m} \sigma_{i} g(z_{i}) \right] = \h \R_{z_1^m}(\sG_k(z_1^m))$, by the
Lemma~\ref{lem:ramon}, for
$\e \geq \frac{\h \R_{z_1^m} (\sG_k(z_1^m))}{\sqrt[\alpha] {2^k/m}}$, the
following holds:
\begin{align*}
& \Pr_{\bsigma} \left[ \sup_{g \in \sG_k(z_1^m)} \frac{1}{m} \sum_{i =
  1}^{m} \sigma_{i} g(z_{i}) > \e\sqrt[\alpha]{\frac{2^{k}}{m}} \,
  \Bigg\mid \, z^{m} \right]\\
& = 
\Pr_{\bsigma} \left[ \sup_{g \in \sG_k(z_1^m)} 
\frac{1}{m} \sum_{i = 1}^{m} \sigma_{i} g(z_{i}) - \h \R_m(\sG_k(z_1^m)) 
> \e \sqrt[\alpha]{\frac{2^{k}}{m}} - \h \R_m(\sG_k(z_1^m)) \, \Bigg\mid \, z^{m} \right]\\
& \leq \exp \left( -\frac{m \left[ \e \sqrt[\alpha]{\frac{2^{k}}{m}} -
      \h \R_{z_1^m} (\sG_k(z_1^m))\right]^2}{2^{k + 5}} \right)
= \exp \left( - \frac{\left(\e - \frac{\h \R_{z_1^m} (\sG_k(z_1^m))}{\sqrt[\alpha]{\frac{2^{k}}{m}}}\right)^2}{32\frac{2^{k(1-2/\alpha)}}{m^{2-2/\alpha}}} \right).
\end{align*}
Since, $-(\e - a)^2 \leq a^2 - \e^2/2$, for $\e \geq
\frac{\h \R_m(\sG_k(z_1^m))}{\sqrt[\alpha] {2^k/m}}$, we can write:
\begin{align*}
\Pr_{\bsigma} \left[ \sup_{g \in \sG_k(z_1^m)} \frac{1}{m} \sum_{i = 1}^{m} \sigma_{i} g(z_{i}) > \e\sqrt[\alpha]{\frac{2^{k}}{m}} \, \Bigg\mid \, z^{m} \right]
& \leq \exp \left(  \frac{\left( \frac{\h \R_m(\sG_k(z_1^m))}{\sqrt[\alpha] {2^k/m}}\right)^2}{32\frac{2^{k(1-2/\alpha)}}{m^{2-2/\alpha}}} \right) \cdot \exp \left( - \frac{\e^2}{64\frac{2^{k(1-2/\alpha)}}{m^{2-2/\alpha}}} \right) \\
& = \exp \left( \frac{m^2\h \R_m^2(\sG_k(z_1^m))}{2^{k+5}} \right) \cdot \exp \left( - \frac{\e^2}{64\frac{2^{k(1-2/\alpha)}}{m^{2-2/\alpha}}} \right).
\end{align*}
For $\e < \frac{\h \R_m(\sG_k(z_1^m))}{\sqrt[\alpha] {2^k/m}}$, the bound holds trivially since the right-hand side is at most one.
\end{proof}

The following is a margin-based relative deviation bound expressed in terms of
Rademacher complexities.

\begin{reptheorem}{thm:rad}
Fix $1 < \alpha \leq 2$. 
Then, with probability at least $1 - \delta$, 
for all hypothesis $h \in \sH$, the following inequality holds:
\[
R(h) - \h R^\rho_{S}(h) 
\leq  16\sqrt{2}  \sqrt[\alpha]{R(h)} \left( \frac{\A_m(\sG) + \log \log m + \log \frac{16}{\delta}}{m}\right)^{1-1/\alpha}.
\]
\end{reptheorem}
\begin{proof}
Let $\A_m^k(\sG)$ be the $k$-peeling-based Rademacher complexity of $\sG$ defined as follows:
\[
\A_m^k(\sG) = \log \E_{z_1^m} \left[\exp \left( \frac{m^2\h \R_m^2(\sG_k(z_1^m))}{2^{k+5}} \right) \right].
\]
Combining Lemmas~\ref{lem:four}, \ref{lem:cover}, \ref{lem:to_rad},
and \ref{lem:indic} yields:
\begin{align*}
& \Pr_{S \sim \sD^{m}} \left[\sup_{h \in \sH} \frac{R(h) - \h
  R^\rho_{S}(h)}{\sqrt[\alpha]{R(h) + \tau}} > \e \right]  \\
 & \leq 8 \Pr_{z_{1}^{m} \sim \sD^{m}, \bsigma} \left[\sup_{g \in \sG} \frac{ \frac{1}{m} \sum_{i = 1}^{m} \sigma_{i} g(z_{i}) }{ \sqrt[\alpha]{\frac{1}{m} [\sum_{i = 1}^{m}
(g(z_{i})) + 1]} } > \frac{\e}{2\sqrt{2}} \right] \\
& = 8 \E_{z^m \sim \sD^m} \left[\Pr_{\bsigma} \left[\sup_{g \in \sG} \frac{ \frac{1}{m} \sum_{i = 1}^{m} \sigma_{i} g(z_{i}) }{ \sqrt[\alpha]{\frac{1}{m} [\sum_{i = 1}^{m}
(g(z_{i})) + 1]} } > \frac{\e}{2\sqrt{2}} \, \Bigg\mid \, z^{m} \right]\right] \\
& \leq 16 \E_{z^m \sim \sD^m} \left[ \sum_k
\exp \left( \frac{m^2\h \R_m^2(\sG_k(z_1^m))}{2^{k+5}} \right) \cdot \exp \left( - \frac{\e^2}{512\frac{2^{k(1-2/\alpha)}}{m^{2-2/\alpha}}} \right)  1_{\e \leq 4\sqrt{2}
\left(\frac{2^{k}}{m} \right)^{1-1/\alpha}} \right] \\
&  = 16  \sum_k
\E_{z^m \sim \sD^m} \left[ \exp \left( \frac{m^2\h \R_m^2(\sG_k(z_1^m))}{2^{k+5}} \right)\right] \cdot \exp \left( - \frac{\e^2}{512\frac{2^{k(1-2/\alpha)}}{m^{2-2/\alpha}}} \right)  1_{\e \leq 4\sqrt{2}
\left(\frac{2^{k}}{m} \right)^{1-1/\alpha}}  \\
&  \leq  16  (\log_2 m )
\E_{z^m \sim \sD^m} \left[ \exp \left( \frac{m^2\h \R_m^2(\sG_k(z_1^m))}{2^{k+5}} \right)\right] \cdot \exp \left( - \frac{\e^2}{512\frac{2^{k(1-2/\alpha)}}{m^{2-2/\alpha}}} \right)  1_{\e \leq 4\sqrt{2}
\left(\frac{2^{k}}{m} \right)^{1-1/\alpha}}  \\
& \leq   16  (\log_2 m)   \sup_k e^{\A_m^k(\sG)} \cdot \exp \left( - \frac{\e^2}{512\frac{2^{k(1-2/\alpha)}}{m^{2-2/\alpha}}} \right)  1_{\e \leq 4\sqrt{2}
\left(\frac{2^{k}}{m} \right)^{1-1/\alpha}}  \\
\end{align*}
Hence, with probability at least $1 - \delta$,
\[
\sup_{h \in \sH} \frac{R(h) - \h
  R^\rho_{S}(h)}{\sqrt[\alpha]{R(h) + \tau}}
  \leq \sup_k \min \left( 16\sqrt{2} \frac{2^{k(1/2-1/\alpha)}}{m^{1-1/\alpha}} \sqrt{\A_m^k(\sG) + \log \log m + \log \frac{16}{\delta}} , 4\sqrt{2}
\left(\frac{2^{k}}{m} \right)^{1-1/\alpha} \right).
\]
For $\alpha \leq 2$, the first term in the minimum decreases with $k$ and the second term
increases with $k$. Let $k_0$ be such that 
\[
2^{k_0} = 16 \left(\sup_k \A_m^k(\sG) + \log \log m + \log \frac{16}{\delta}\right) = 16 \left(\A_m(\sG) + \log \log m + \log \frac{16}{\delta} \right).
\]
Then for any $k$, 
\begin{align*}
&  \sup_k \min \left( 16\sqrt{2} \frac{2^{k(1/2-1/\alpha)}}{m^{1-1/\alpha}} \sqrt{\A_m^k(\sG) + \log \log m + \log \frac{16}{\delta}} , 4\sqrt{2}
\left(\frac{2^{k}}{m} \right)^{1-1/\alpha} \right) \\ 
& \leq \sup_k \max \left( 16\sqrt{2} \frac{2^{k_0(1/2-1/\alpha)}}{m^{1-1/\alpha}} \sqrt{\A_m^k(\sG) + \log \log m + \log \frac{16}{\delta}}, 4 \sqrt{2}
\left(\frac{2^{k_0}}{m} \right)^{1-1/\alpha} \right)\\
& \leq \max \left(16 \sqrt{2} \frac{2^{k_0(1/2-1/\alpha)}}{m^{1-1/\alpha}} \sqrt{\A_m(\sG) + \log \log m + \log \frac{16}{\delta}}, 4 \sqrt{2}
\left(\frac{2^{k_0}}{m} \right)^{1-1/\alpha} \right) \\
& \leq 4\sqrt{2} \left(\frac{2^{k_0}}{m} \right)^{1-1/\alpha}  \\
& \leq  16\sqrt{2} \left( \frac{\A_m(\sG) + \log \log m + \log \frac{16}{\delta}}{m} \right)^{1-1/\alpha}.
\end{align*}
Rearranging and taking the limit as $\tau \to 0$ yields the result.
\end{proof}

\begin{lemma}
\label{lem:ineq}
For any $x, y, z \geq 0$, if 
$(x - y \sqrt[\alpha]{x} \leq z)$,
then the following inequality holds:
\[
x \leq z +  2y \sqrt[\alpha]{z} + (2 y)^{\frac{\alpha}{\alpha-1}}.
\]
\end{lemma}
\begin{proof}
In view of the assumption, we can write:
\[
x \leq z + y \sqrt[\alpha]{x} \leq 2 \max(z,    y \sqrt[\alpha]{x}),
\]
If $z \geq y \sqrt[\alpha]{x}$, then $x \leq 2z$. if
$z \leq y \sqrt[\alpha]{x}$, then $x \leq (2
y)^{\alpha/(\alpha-1)}$. This shows that we have
$x \leq 2 \max(z, (2y)^{1-1/\alpha})$. Plugging in the right-hand side
in the previous inequality and using the sub-additivity of
$x \mapsto \sqrt[\alpha]{x}$ gives:
\[
x \leq z + y \sqrt[\alpha]{x}
\leq  z + y \sqrt[\alpha]{2 \max(z, (2y)^{\alpha/(\alpha-1)})} \leq
z + y \sqrt[\alpha]{2z} + y^{\frac{\alpha}{\alpha - 1}} 2^{\frac{1}{\alpha} + \frac{1}{\alpha - 1}}.
\]
The lemma follows by observing that $2^{\frac{1}{\alpha}} \leq 2$ for $\alpha \geq 1$.
\end{proof}

\begin{repcorollary}{cor:all_alpha}
Let $\sG$ be defined as above. Then, with probability at least $1 - \delta$, for all
hypothesis $h \in \sH$ and $\alpha \in (0, 1]$,
\[
R(h) - \h
  R^\rho_{S}(h) \leq  32\sqrt{2}  \sqrt[\alpha]{R(h) }  \left(  \frac{\A_m(\sG) + \log \log m + \log \frac{16}{\delta}}{m}\right)^{1-1/\alpha}.
\]
\end{repcorollary}
\begin{proof}
By Theorem~\ref{thm:rad},
\[
R(h) - \h
  R^\rho_{S}(h) \leq  16  \sqrt[\alpha]{R(h) }  \left(  \frac{\A_m(\sG) + \log \log m + \log \frac{16}{\delta}}{m}\right)^{1-1/\alpha}.
\]
Let $B = \A_m(\sG) + \log \log m + \log \frac{16}{\delta}$.  Let
$\alpha_k = 1+e^{-\epsilon k}$. Let $\delta_k = \delta/k^2$. Then, by
the union bound, for all $\alpha_k$, with probability at least 
$1 - \delta$,
\[
R(h) - \h
  R^\rho_{S}(h) \leq  16 \sqrt{2} \sqrt[\alpha_k]{R(h) }  \left(  \frac{B + 2\log k}{m}\right)^{1-1/\alpha_k}.
\]
Let $\alpha_{k} \geq \alpha \geq \alpha_{k+1}$. 
Then $(k+1) \leq  \frac{1}{\epsilon} \log \frac{1}{\alpha - 1}$.
Then,
\begin{align*}
&  \sqrt[\alpha]{R(h) }  \left(  \frac{B +\log \frac{1}{\alpha - 1} }{m}\right)^{1-1/\alpha}  \\
&  \sqrt[\alpha]{R(h) }  \left(  \frac{B + 2\log (k+1)}{m}\right)^{1-1/\alpha} \\
& \geq \min \left( \sqrt[\alpha_k]{R(h) }  \left(  \frac{B + 2\log (k+1)}{m}\right)^{1-1/\alpha_k}, \sqrt[\alpha_{k+1}]{R(h) }  \left(  \frac{B + 2\log (k+1)}{m}\right)^{1-1/\alpha_{k+1}} \right).
\end{align*}
Hence, with probability at least $1 - \delta$, for all $\alpha \in (1,
2]$, 
\[
R(h) - \h
  R^\rho_{S}(h) \leq  16 \sqrt{2} \sqrt[\alpha]{R(h)}  \left(  \frac{B + 2 \log \frac{1}{\alpha - 1}}{m}\right)^{1-1/\alpha}.
\]
The lemma follows by observing that 
\[
 \left(  \frac{B + 2\log \frac{1}{\alpha - 1}}{m}\right)^{1-1/\alpha}
 \leq  \left(  \frac{B}{m}\right)^{1-1/\alpha} +  \left( 2 \frac{\log \frac{1}{\alpha - 1}}{m}\right)^{1-1/\alpha}
 \leq \left(  \frac{B}{m}\right)^{1-1/\alpha} +  \left(  \frac{1}{m}\right)^{1-1/\alpha} \leq 2\left(  \frac{B}{m}\right)^{1-1/\alpha} .
\]
\end{proof}

\newpage
\section{Upper bounds on peeling-based Rademacher complexity}
\label{app:peel}

\begin{replemma}{lem:card}
For any class $\sG$,
\[
\A_m(\sG) \leq \frac{1}{8} \log \E_{z_1^m} [\Sm_{\sG}({z_{1}^{m}}].
\]
\end{replemma}
\begin{proof}
By definition,
\[
\A_m(\sG) = \sup_k \log \E_{z_1^m} \left[\exp \left( \frac{m^2\h \R_m^2(\sG_k(z_1^m))}{2^{k+5}} \right) \right].
\]
For any $g \in \sG_k(z_1^m)$, since $g$ takes values in $[0, 1]$, we have:
\[
\sum_{i = 1}^m g^2(z_i) \leq \sum_{i = 1}^m g(z_i) \leq \frac{2^{k + 1}}{m}.
\]
Thus, by Massart's lemma and Jensen's inequality, the following
inequality holds:
\[
  \h \R_m(\sG_k(z_1^m)) \leq \sqrt{2 \log \E_{z_1^m} [|\sG_k(z_1^m)|]}
  \sqrt{\frac{2^{k+1}}{m}} \leq \sqrt{2 \log \E_{z_1^m} [\Sm_{\sG}({z_{1}^{m}})]}
  \sqrt{\frac{2^{k+1}}{m^2}}.
\]
Hence,
\[
\A_m(\sG) \leq  \sup_k \frac{1}{2^3} \log \E_{z_1^m} [\Sm_{\sG}({z_{1}^{m}}] 
= \frac{1}{8} \log \E_{z_1^m} [\Sm_{\sG}({z_{1}^{m}}].
\]
\end{proof}
\begin{replemma}{lem:covering}
For a set of hypotheses $\sG$,
\[
\A_m(\sG) \leq \sup_{0 \leq k \leq \log_2 (m)} \log \left[ \E_{z_1^m \sim
    \sD^m} \left[\exp \left( \frac{1}{16} \left(1 + \int^{1}_{\epsilon = 1/\sqrt{m}}
\log N_2(\sG_k(z_1^m), \epsilon \sqrt{2^k/m}) d \epsilon  \right)\right) \right] \, \right].
\]
\end{replemma}
\begin{proof}
By Dudley's integral,
\[
\h \R_m(\sG_k(z_1^m))
= \min_{\tau}  \tau + \int^{2^k/m}_{\epsilon = \tau}
\sqrt{\frac{\log N_2(\sG_k(z_1^m), \epsilon)}{m}} d \epsilon.
\]
Choosing $\tau = \frac{2^{k/2}}{m}$ and changing variables from $\epsilon$ to $\epsilon \frac{2^{k/2}}{\sqrt{m}}$ yields,
\[
\h \R_m(\sG_k(z_1^m))
= \frac{2^{k/2}}{m}+ \frac{2^{k/2}}{m}\int^{1}_{\epsilon = 1/\sqrt{m}}
\sqrt{\log N_2(\sG_k(z_1^m), \epsilon \sqrt{2^k/m})} d \epsilon.
\]
Using $(a + b)^2 \leq 2 a^2 + 2b^2$ and the Cauchy-Schwarz inequality yields,
\begin{align*}
\frac{m^2\h
        \R_m^2(\sG_k(z_1^m))}{2^{k+5}} 
     &   \leq \frac{1}{16} \left(1 + \left(\int^{1}_{\epsilon = 1/\sqrt{m}}
\sqrt{\log N_2(\sG_k(z_1^m), \epsilon \sqrt{2^k/m})} d \epsilon \right)^2 \right) \\
   &   \leq \frac{1}{16} \left(1 + \int^{1}_{\epsilon = 1/\sqrt{m}}
\log N_2(\sG_k(z_1^m), \epsilon \sqrt{2^k/m}) d \epsilon  \right).
\end{align*}
\end{proof}

Recall that the worst case Rademacher complexity is defined as follows.
\[
\h \R_m^{\max} (\sH) = \sup_{z^m_1} \h \R_{m} (\sH)
\]
\begin{replemma}{lem:smooth}
Let $g$ be the smoothed margin loss from~\citep[Section 5.1]{SrebroSridharanTewari2010},
with its second moment is bounded by $\pi^2/4\rho^2$.
then 
\[
\A_m(\sG) \leq \frac{16 \pi^2 m}{\rho^2} ( \h \R^{\max}_m (\sH))^2
 \left(2 \log^{3/2} \frac{m}{\h \R_m^{\max} (\sH)} 
- \log^{3/2} \frac{2\pi m}{ \rho\h \R_m^{\max} (\sH)}\right)^2.
\]
\end{replemma}
\begin{proof}
Recall that the smoothed margin loss of~\cite{SrebroSridharanTewari2010} is given by
\begin{equation}
  g(yh(x)) =
    \begin{cases}
      1 & \text{if } yh(x) < 0 \\
      \frac{1 + \cos(\pi yh(x)/ \rho)}{2} & \text{if } yh(x) \in [0, \rho] \\
      0 & \text{if } yh(x) > \rho.
    \end{cases}       
\end{equation}
Upper bounding the expectation by the maximum gives:
\[
\A_m(\sG) \leq  \sup_k \log \sup_{z^m_1} \left[\exp \left( \frac{m^2\h \R_m^2(\sG_k(z_1^m))}{2^{k+5}} \right) \right] \leq 
\sup_k \sup_{z^m_1}   \frac{m^2\h \R_m^2(\sG_k(z_1^m))}{2^{k+5}}.
\]
Let $\sG'_k(z_1^m) = \set[\Big]{g \in \sG\colon \sum_{i = 1}^m g(z_i) + 1  \leq
  2^{k + 1}}$. Since $\sG_k(z_1^m) \subseteq \sG'_k(z^m)$,
\[
\A_m(\sG) \leq  \sup_k \sup_{z^m_1}   \frac{m^2\h \R_m^2(\sG'_k(z^m))}{2^{k+5}}.
\]
Now, $\h \R_m(\sG'_k(z^m))$ coincides with the local Rademacher
complexity term defined in \citep [Section
2]{SrebroSridharanTewari2010}. Thus, by \citep [Lemma
2.2]{SrebroSridharanTewari2010},
\[
\h \R_m(\sG'_k(z^m)) \leq \frac{16\pi}{\rho} \h \R_m^{\max} (\sH)
\sqrt{\frac{2^{k+1}}{m}} \left(2 \log^{3/2} \frac{m}{\h \R_m^{\max} (\sH)} -   \log^{3/2} \frac{2\pi m}{ \rho\h \R_m^{\max} (\sH)}\right).
\]
\end{proof}

\newpage
\section{Unbounded margin losses}
\label{app:ubound}
\begin{reptheorem}{th:unbound} Fix $\rho \geq 0$. Let $1 < \alpha \leq 2$, $0 < \epsilon \leq 1$, and $0 < \tau^{\frac{\alpha - 1}{\alpha}} < \epsilon^{\frac{\alpha}{\alpha - 1}}$. For any loss function $L$ (not necessarily
bounded) and hypothesis set $H$ such that $\L_{\alpha}(h) < +\infty$ for all $h \in H$, 
\begin{multline*}
\Pr \left[\sup_{h \in H} \L(h) - \h \L_{S}(h)\,  > {\Gamma}_{\tau}(\alpha, \epsilon)\, \epsilon {\sqrt[\alpha]{\L_{\alpha}(h) + \tau}}+ \rho
\right] \\
\leq \Pr \left[\sup_{h \in H, t \in \Rset} 
\frac{\Pr[L(h, z) > t] - \h \Pr[L(h, z) > t - \rho]}{{\sqrt[\alpha]{\Pr[L(h, z) > t] + \tau}}}
 > \epsilon\right],
\end{multline*}
where ${\Gamma}_{\tau}(\alpha, \epsilon) = \frac{\alpha - 1}{\alpha} (1 + \tau)^{\frac{1}{\alpha}} + \frac{1}{\alpha} \left( \frac{\alpha}{\alpha - 1} \right)^{\alpha - 1} (1 + \left(\frac{\alpha -
1}{\alpha}\right)^{\alpha} \tau^{\frac{1}{\alpha}})^{\frac{1}{\alpha}} \left[ 1 + \frac{\log (1/\epsilon)}{\left( \frac{\alpha}{\alpha - 1} \right)^{\alpha - 1}} \right]^{\frac{\alpha - 1}{\alpha}}$
\ignore{ ${\Gamma}_{\tau}(\beta, \epsilon) = \frac{1}{\beta} + \left( \frac{\beta - 1}{\beta} \right) \beta^{\frac{1}{\beta - 1}} \left[1 + \frac{\log(1/\epsilon)}{\beta^{\frac{1}{\beta - 1}}}
\right]^{\frac{1}{\beta}}$, with $\frac{1}{\alpha} + \frac{1}{\beta} = 1$}.
\end{reptheorem}
\begin{proof}
Fix $ 1 < \alpha \le 2$ and $\epsilon > 0$ and $\mathcal{S}$ assume that for any $h \in H$ and $t \geq 0$, the following holds:
\begin{equation}
\label{eq:14} \frac{\Pr[L(h, z) > t ] - \widehat{\Pr}[L(h, z) > t - \rho] }{\sqrt[\alpha]{\Pr[L(h, z) > t] + \tau}} \leq \epsilon.
\end{equation}
Let $t_1 = \frac{\alpha - 1}{\alpha} \sqrt[\alpha]{\mathcal{L}_{\alpha}(h) + \tau}\left[ \frac{1}{\epsilon}
\right]^{\frac{1}{\alpha - 1}}$. We show that this implies that for any $h \in H$, $\mathcal{L}(h) - \widehat{\mathcal{L}}_{S}(h)\leq {\Gamma}_{\tau}(\alpha, \epsilon) \epsilon {\sqrt[\alpha]{\mathcal{L}_{\alpha}(h) + \tau}}  + \min(\rho, t_1)$. By the properties of
the Lebesgue integral, we can write
\begin{eqnarray*}
&& \mathcal{L}(h) = \mathrm{E}_{z \sim D}[L(h, z)] = \int_{0}^{+\infty} \Pr[L(h, z) > t]\, dt.
\end{eqnarray*}
Similarly, we can write
\begin{align*}
 \widehat{\mathcal{L}}(h) = \mathrm{E}_{z \sim \widehat{D}}[L(h, z)] &= \int_{0}^{+\infty} \widehat{\Pr}[L(h, z) > u]\, du \\
 & = \int_{\rho}^{+\infty} \widehat{\Pr}[L(h, z) > t - \rho]\, dt \\
 & = \int_{0}^{+\infty} \widehat{\Pr}[L(h, z) > t - \rho]\, dt - \int_{0}^{\rho} \widehat{\Pr}[L(h, z) > t - \rho]\, dt \\
\text{and} \quad \mathcal{L}_{\alpha}(h) 
& = \int_{0}^{+\infty} \Pr[L^{\alpha}(h, z) > t]\, dt = \int_{0}^{+\infty} \alpha t^{\alpha - 1} \Pr[L(h, z) > t]\, dt.
\end{align*}
To bound $\mathcal{L}(h) - \widehat{\mathcal{L}}(h)$, we simply bound $\Pr[L(h, z) > t] - \widehat{\Pr}[L(h, z) > t - \rho]$ by $\Pr[L(h, z) > t]$ for large values of $t$, that
is $t > t_{1}$, and use inequality (\ref{eq:14}) for smaller values of $t$:
\begin{eqnarray*}
& = & \mathcal{L}(h) - \widehat{\mathcal{L}}(h) \\
& = & \int_{0}^{+\infty} \Pr[L(h, z) > t] - \widehat{\Pr}[L(h, z) > t-\rho]\, dt + \int_{0}^{\rho} \widehat{\Pr}[L(h, z) > t - \rho]\, dt\\
& \leq & \int_{0}^{+\infty} \Pr[L(h, z) > t] - \widehat{\Pr}[L(h, z) > t-\rho]\, dt + \rho\\
& \leq & \int_{0}^{t_{1}} \epsilon \sqrt[\alpha]{\Pr[L(h, z) > t] + \tau} \, dt + \int_{t_{1}}^{+\infty} \Pr[L(h, z) > t] \, dt + \min(t_1, \rho),
\end{eqnarray*}
where the last two inequalities use the fact that 
$L$ is non-negative.
The rest of the proof is similar to \cite[Theorem 3]{CortesGreenbergMohri2019}.
\end{proof}

\begin{corollary}
\label{cor:ubound_all_rho}  Let $\epsilon < 1$, $1 < \alpha \le 2$.
and hypothesis set $\sH$ such that $\L_{\alpha}(h) < +\infty$ for all $h \in \sH$,
\begin{align*}
  \L(h) - \h \L_{S}(h) 
& \leq \min_{\rho \leq r} \gamma
  \sqrt[\alpha]{\L_{\alpha}(h)} 
  \sqrt{\frac{\log \E[\cN_\infty(\L(\sH), \tfrac{\rho}{2}, x_1^{2m})] + \log \frac{1}{\delta} + \log \log \frac{2r}{\rho}}{m^{\frac{2 (\alpha - 1)}{\alpha}}}} + \rho,
\end{align*}
where $\gamma = \Gamma_0\left(\alpha,   \sqrt{\frac{\log \E[\cN_\infty(\L(\sH), \tfrac{\rho}{2}, x_1^{2m})] + \log \frac{1}{\delta} + \log \log \frac{2r}{\rho}}{m^{\frac{2 (\alpha - 1)}{\alpha}}}}
\right) = \mathcal{O}(\log m)$.
\end{corollary}
The proof of Corollary~\ref{cor:ubound_all_rho} is similar to that of Corollary~\ref{cor:4}
and is omitted.

\newpage
\section{Applications}
\label{app:applications}

\subsection{Algorithms}

As discussed in Section~\ref{sec:applications}, our results can help
derive tighter guarantees for margin-based algorithms such as Support
Vector Machines (SVM) \citep{CortesVapnik1995} and other algorithms
such as those based on neural networks that can be analyzed in terms
of their margin.
But, another potential application of our learning bounds is to design
new algorithms, either by seeking to directly minimize the resulting
upper bound, or by using the bound as an inspiration for devising a
new algorithm.

In this sub-section, we briefly initiate this study in the case
of linear hypotheses. We describe an algorithm seeking to minimize the
upper bound of Corollary~\ref{cor:3} (or
Corollary~\ref{cor:smooth_alpha}) in the case of linear hypotheses.
Let $R$ be the radius of the sphere containing the data. Then, the
bound of the corollary holds with high probability for any function
$h\colon \bx \mapsto \bw \cdot \bx$ with $\bw \in \Rset^d$,
$\| \bw \|_2 \leq 1$, and for any $\rho > 0$ for $d =
(R/\rho)^2$. Ignoring lower order terms and logarithmic factors, the
guarantee suggests seeking to choose $\bw$ with $\| \bw \| \leq 1$ and
$\rho > 0$ to minimize the following:
\begin{align*}
\h R^{\rho}_{S}(\bw) +  \frac{\lambda}{\rho} \sqrt{\h R^{\rho}_{S}(\bw)},
\end{align*}
where we denote by $\h R^{\rho}_{S}(\bw)$ the empirical margin
loss of $h\colon \bx \mapsto \bw \cdot \bx$. 
Thus, using the so-called ramp loss
$\Phi_\rho\colon u \mapsto \min (1, \max(0, 1 - \frac{u}{\rho}))$,
this suggests choosing $\bw$ with $\| \bw \| \leq 1$ and $\rho > 0$ to
minimize the following:
\begin{align*}
\frac{1}{m} \sum_{i = 1}^m \Phi_\rho(y_i \bw \cdot \bx_i) 
+ \frac{\lambda}{\rho} \sqrt{\frac{1}{m} \sum_{i = 1}^m 
\Phi_\rho(y_i \bw \cdot \bx_i) }.
\end{align*}
This optimization problem is closely related to that of SVM but it is
distinct. The problem is non-convex, even if $\Phi_\rho$ is upper
bounded by the hinge loss. The solution may also not coincide with
that of SVM in general.
As an example, when the training sample is linearly separable, any
pair $(\bw^*, \rho^*)$ with a weight vector $\bw^*$ defining a
separating hyperplane and $\rho^*$ sufficiently large is solution,
since we have
$\sum_{i = 1}^m \Phi_{\rho^*}(y_i \bw^* \cdot \bx_i) = 0$.  In
contrast, for (non-separable) SVM, in general the solution may not be
a hyperplane with zero error on the training sample, even when the
training sample is linearly separable. Furthermore, the SVM solution
is unique \citep{CortesVapnik1995}.

\subsection{Active learning}

Here, we briefly highlight the relevance of our learning bounds to the
design and analysis of active learning algorithms. One of the key
learning guarantees used in active learning is a standard relative
deviation bound. This is because scaled \emph{multiplicative bounds}
can help achieve a better label complexity.

Many active learning algorithms such as DHM
\citep{DasguptaHsuMonteleoni2008} rely on these bounds.
However, as pointed out by the authors, the empirical error
minimization required at each step of the algorithm is NP-hard for
many classes, for example linear hypothesis sets. To be precise, the
algorithm requires a hypothesis consistent with sample A, with minimum
error on sample B.  That requires hard constraints corresponding to
every sample in A.
An open question raised by the authors is whether a
margin-maximization algorithm such as SVM can be used instead, while
preserving generalization and label complexity guarantees
(\citep[section 3.1, p.~5]{DasguptaHsuMonteleoni2008}).

To do so, the key lemma used by the authors for much of their proofs
needs to be extended to the empirical margin loss case
\citep[Lemma~1]{DasguptaHsuMonteleoni2008}.  That lemma is precisely
the relative deviation bounds for the zero-one loss case
\citep{Vapnik1998,Vapnik2006,AnthonyShaweTaylor1993,
CortesGreenbergMohri2019}. Using
a notation similar to the one adopted by
\cite{DasguptaHsuMonteleoni2008}, the extension to the empirical
margin loss case of that lemma would have the following form:
\begin{align*}
 R(h) - \h R_S^\rho(h) 
\leq \min \left\{ \alpha_m \sqrt{\h R_S^\rho(h) } +
  \alpha_m^2, \alpha_m \sqrt{R(h)} \right\}.
\end{align*}
This is precisely the results shown in Theorem~\ref{th:relative} and
Corollary~\ref{cor:2}, which hold with probability at least
$1 - \delta$ for all $h \in \sH$, for
$\alpha_m = 2\sqrt{ \frac{\log \E[\cN_\infty(\sH_\rho, \frac{\rho}{2},
    x_1^{2m})] + \log \frac{1}{\delta}}{m}}$. Similar results can also
be shown using our Rademacher complexity bounds of
Section~\ref{sec:rademacher}.